\documentclass{article}

\usepackage[nonatbib,preprint]{neurips_2026}

\usepackage{amsmath,amsfonts,bm}

\def\eqref#1{equation~\ref{#1}}

\def\1{\bm{1}}

\DeclareMathAlphabet{\mathsfit}{\encodingdefault}{\sfdefault}{m}{sl}
\SetMathAlphabet{\mathsfit}{bold}{\encodingdefault}{\sfdefault}{bx}{n}

\usepackage[utf8]{inputenc}
\usepackage[
  style=numeric,
  sorting=ynt,
  sortcites=true,
  maxbibnames=99
]{biblatex}
\usepackage[T1]{fontenc}
\usepackage{xcolor}
\usepackage{graphicx}
\usepackage{amsthm}
\usepackage{float}
\usepackage{amsmath}
\usepackage{subcaption}
\usepackage{url}
\usepackage{booktabs}
\usepackage{amsfonts}
\usepackage{nicefrac}
\usepackage{microtype}

\usepackage[colorlinks=true, linkcolor=blue, citecolor=blue, urlcolor=blue]{hyperref}
\usepackage{cleveref}

\title{Refusal geometry reflects refusal training: diverse refusal prefixes can raise stable rank and weaken refusal vector ablation attacks}

\author{%
  Andrey Labunets \\
  UC San Diego\\
  \texttt{alabunets@ucsd.edu}
}

\newsavebox{\infographicAbox}

\newtheorem{theorem}{Theorem}[section]
\newtheorem{lemma}[theorem]{Lemma}
\newtheorem{corollary}[theorem]{Corollary}

\theoremstyle{definition}
\newtheorem{definition}[theorem]{Definition}

\theoremstyle{remark}

\newcommand{\RR}{\mathbb{R}}

\newcommand{\rank}{\operatorname{rank}}
\newcommand{\sr}{\operatorname{sr}}

\begin{document}

\maketitle

\begin{abstract}
    Refusal training protects AI models from jailbreaks by training models to decline unsafe queries, reducing the risk of misuse.
    Recent work finds that refusal behavior in aligned language models can be mediated by a single activation direction or a low-dimensional refusal subspace shared across harmful prompts: ablating those directions suppresses refusals while largely preserves other model capabilities.
    Yet it remains unclear why safety-critical features in a wide range of models emerge in a concentrated, low-dimensional structure.
    In a case study of OLMo-2-0425-1B-Instruct we find that the refusal geometry reflects refusal training: activation updates resulting from refusal-completion first-token losses explain the resulting refusal direction and refusal subspace.
    We study refusal directions through the training dynamics across refusal datasets and reveal that their brittleness is associated with repetitive refusal starts, which in turn is linked to concentration of gradients and refusal features in a low-dimensional subspace.
    Across frozen-model analyses and controlled synthetic fine-tuning, we find evidence of a hardening lever: diverse refusal starts can raise stable ranks of gradients and activation changes, making refusals harder to remove with a vector ablation attack.
\end{abstract}

\section{Introduction}

State-of-the-art AI systems are beginning to exhibit dual-use capabilities that could pose severe risks if misused, such as 
exploit generation \parencite{anthropic2026mythospreview}, or assistance with biological threat creation or bio-weaponization \parencite{openai2025futurebiology,gopal2023willreleasingweights}.
This motivates the now-standard view that, as models become more capable, protections should prevent harmful outputs and catastrophic misuse of AI \parencite{bengio2026internationalaisafetyreport,anderljung2023frontierairegulation}.

AI safety and security study ways to reduce these risks, which is typically done with multi-layered  approaches.
Frontier-safety frameworks emphasize capability evaluations, confidentiality of model weights, deployment-time mitigations, monitoring, access controls, and broader ecosystem-level defenses \parencite{bengio2026internationalaisafetyreport,anderljung2023frontierairegulation,buhl2025emergingpracticesfrontier,metr2025commonelements,frontiermodelforum2025frontiermitigations}.
Model-weight security is important because weight access can allow protections to be modified or removed; at the same time, weights access scenario is realistic for open models, insider compromise, model theft, and sufficiently capable adversaries  \parencite{nevo2024securingaimodelweights,carlini2024stealingproductionlanguagemodel,gopal2023willreleasingweights}.

A central technical layer of protection for deployed models, including open-weights models, is safety alignment: models are trained to remain helpful on benign requests while refusing instructions that request harmful assistance or dangerous capabilities \parencite{ouyang2022traininglanguagemodels,bai2022constitutional,qi2024safetyalignmentjusttokens}.
These protections are also evaluated against jailbreak attacks that aim to bypass safety alignment \parencite{perez2022redteaming,wei2023jailbroken,zou2023universal,chao2023jailbreakingblackbox,mehrotra2023tree}.
This safety-aligned behavior is often induced by training on harmful prompts paired with safe refusal completions.

Recent evidence, however, shows that in local, white-box settings those \emph{refusals} in aligned language models can be mediated by a low-dimensional mechanism: \emph{refusal direction}, or a single vector in representation space, projections on which strongly control refuse/comply behavior \parencite{arditi2024refusallanguagemodelsmediated}.
More broadly, activation-steering methods construct directions for specified concepts or behaviors, using prompt pairs, contrastive datasets, supervised
concept labels, or optimized target completions to control properties such as topic, sentiment, factuality, deception, or safety-relevant behavior  \parencite{dathathri2020plugplaylanguagemodels, subramani2022extractinglatentsteeringvectors, turner2024steeringlanguagemodelsactivation, panickssery2024steeringllama2contrastive, beaglehole2025universalsteeringmonitoringai, dunefsky2025oneshotoptimizedsteeringvectors}.
Refusal ablation, however, is a particularly important special case because it targets the shared refuse/comply response mechanism itself.

Existing research mechanistically studied representations of refusal and their geometry, showing that refusal concepts span multiple orthogonal and independent directions, where auxiliary principal components mediate distinct safety-relevant features \parencite{wollschläger2025geometryrefusallargelanguage, pan2025hiddendimensionsllmalignment}.
In particular, \parencite{pan2025hiddendimensionsllmalignment} shows that safety-aligned behavior is controlled by a low-dimensional space.
We use \emph{refusal subspace} to denote such an affine subspace of refusal-related activations.
A separate line of work studied forms of concentration leading to low-rank or low-effective-dimensionality representations, including shallow concentration on early safety tokens \parencite{qi2024safetyalignmentjusttokens}, low-stable-rank transformer activations \parencite{davis2026spectralgradientupdateshelp}, rank collapse in pure attention \parencite{dong2021attentionnot}, neural collapse in final and intermediate classifier representations \parencite{papyan2020prevalence,rangamani2023feature}, and regularities in refusals \parencite{prakash2025imsorryicant}.
However, the explanation of why meaningful, safety-critical representations of refusal across a broad range of harmful prompts and topics collapse to a single effective direction or a low-dimensional subspace remained elusive.

In this work, we bridge this gap by studying robustness to refusal ablation attacks analytically and empirically in OLMo-2-0425-1B-Instruct:

\begin{figure}[h]
    \centering
    \subcaptionbox{Low-dim refusal residuals}[0.31\linewidth]{
        \includegraphics[width=0.95\linewidth]{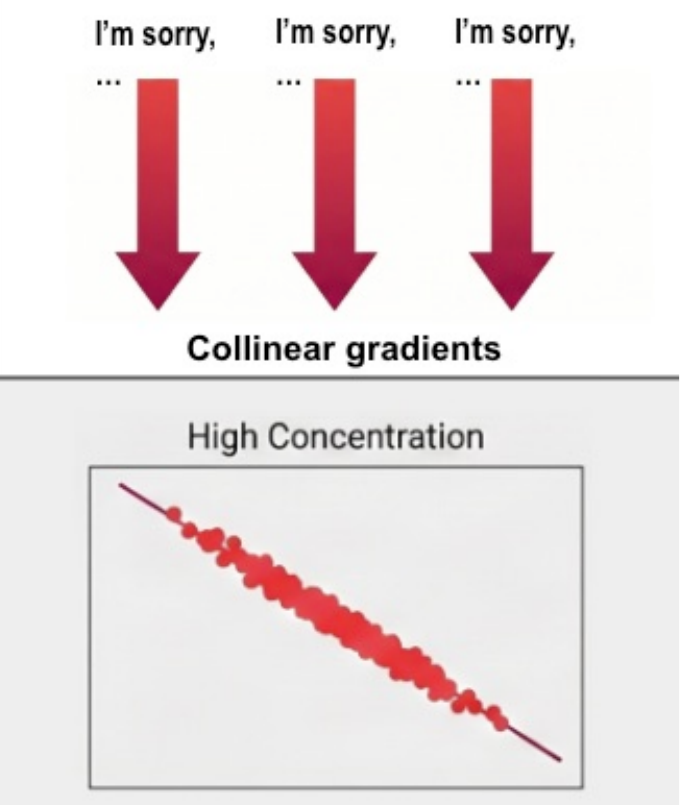}
    }\hfill
    \subcaptionbox{Diverse refusal residuals}[0.31\linewidth]{
        \includegraphics[width=0.95\linewidth]{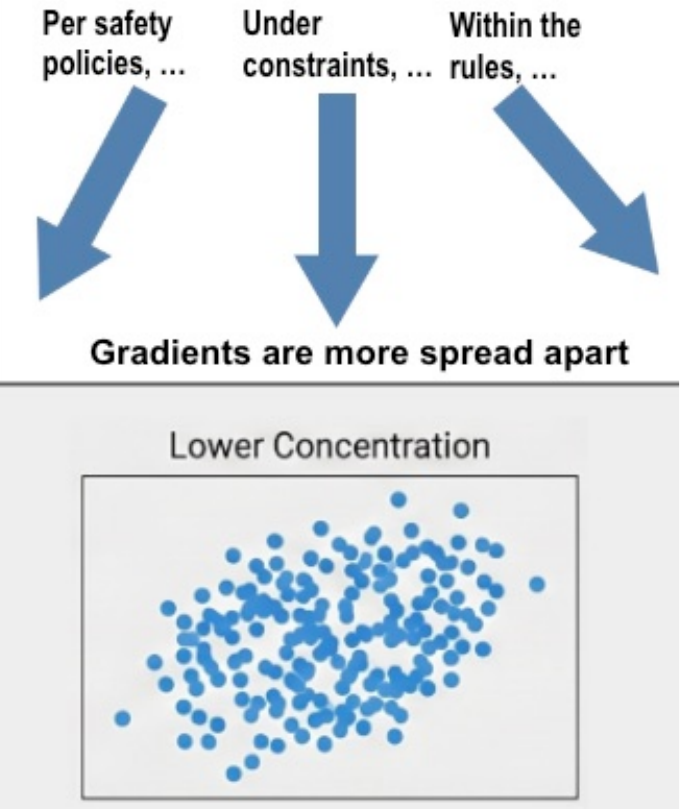}
    }\hfill
    \subcaptionbox{Refusal ablation attack}[0.30\linewidth]{
        \includegraphics[width=\linewidth]{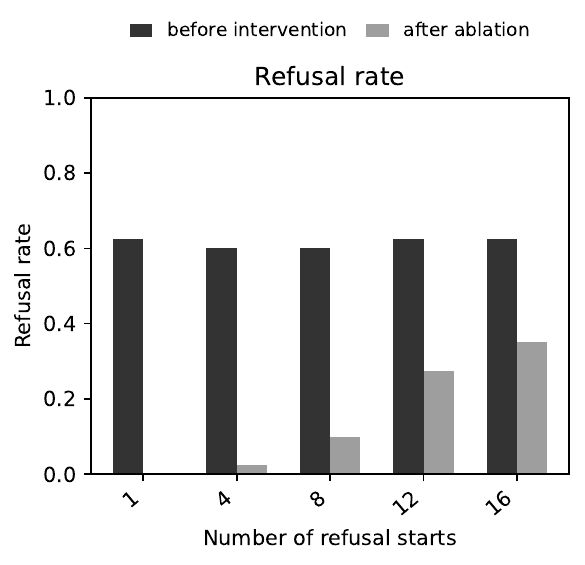}
    }
    \caption{Higher stable rank resists refusal ablation (conceptual schematic). When residuals are less concentrated, a shared single-vector ablation cannot fully disable the refusal mechanism, so the refusal score after attack remains higher.}
    \label{fig:infographic}
\end{figure}
\Cref{fig:infographic} illustrates our hypothesized mechanism: concentrated update directions produce refusal residuals with lower effective dimensionality, making refusal more vulnerable to a shared single-vector ablation.
Concretely, in these settings, we show the following contributions:
\begin{itemize}
    \item \textbf{Training-origin geometry.}
    We show that cross-entropy losses on the first tokens of harmful prompts' completions induce activation changes with partial geometric correspondence to the refusal subspace: their means exhibit completion-dependent signed alignment with the refusal direction, and their principal directions partially overlap direction space of refusal subspace. For refusal completions, the per-gradient-step activation-change and final refusal-residual matrices also have comparable low stable rank.
    \item \textbf{Prefix concentration and spectral collapse.} We identify concentrated refusal-start support in safety datasets --- especially repeated first tokens --- as one mechanism that can produce low-stable-rank gradients, activation changes, and refusal residuals.
    \item \textbf{Cross-layer stable rank transfer.} We introduce stable rank transfer factors and show that refusal-diversity-induced stable rank increases propagate back through late layers with bounded ratios across evaluated conditions.
    \item \textbf{Shared effective Jacobian explanation.}
    We define top-\(k\) source-subspace shared maps and a relative
    transfer factor fit error that yields data-dependent two-sided
    stable-rank bounds on the gradients and gradient-induced activation changes.
    On the fitted supports, these maps reconstruct a meaningful though
    incomplete component of their cross-layer relation and account for
    its observed rank amplification or contraction.
    \item \textbf{Rank--vulnerability connection.}
    Across frozen-model analyses and controlled fine-tuning,
    greater refusal-start diversity is associated with higher
    stable ranks and weaker difference-in-means single-vector ablation, providing evidence of a simple hardening lever.
\end{itemize}

\section{Background and definitions}
\label{sec:background}

\subsection{Transformer architecture and notation}
\label{subsec:background_transformer}

We consider a decoder-only transformer language model for next-token prediction with vocabulary $V=\{1,\dots,|V|\}$, hidden size $d$, and $l \in \{0,\dots,L\}$ layers (transformer blocks).
Training examples are pairs token sequences, prompt $x$ and a completion, with tokens from $V$.
We index training examples by $i$.
We use \emph{activation} to denote a column vector of transformer hidden state at a specified layer and token position.
For example \(i\), let \(h_i^{(l)}\in\mathbb R^d\) denote the layer-\(l\) activation at the first-assistant-token prediction position, while, more broadly, all activations across $N$ training examples at layer $l$ and the same final position are denoted as a matrix of transposed activation column vectors:
\begin{align}
H^{(l)} =
\begin{bmatrix}
h^{(l)\top}_{1} \\
\vdots \\
h^{(l)\top}_{N}
\end{bmatrix}
\in \mathbb{R}^{N \times d}.
\end{align}
The model classifier head (unembedding) is parametrized by a matrix
$W\in\mathbb{R}^{d\times |V|}$ (which consists of embedding vectors $w_1, ..., w_{|V|}$) and bias $b\in\mathbb{R}^{|V|}$, yielding logits $o$,
a probability distribution $p$, and a sampled token $y\in V$ with probability $p(y)$:
\begin{align}
  o &= W^\top h^{(L)} + b \in \mathbb{R}^{|V|}, \\
  p &= \mathrm{softmax}(o)\in \Delta^{|V|-1}
      = \Big\{p\in\mathbb{R}^{|V|}:\ p\ge 0,\ \sum_{k=1}^{|V|} p_k = 1\Big\},\qquad
  p_k = \frac{\exp(o_k)}{\sum_{j=1}^{|V|}\exp(o_j)}
\end{align}
All activations in this paper used for refusal-vector estimation are evaluated at the same position, which is the first-assistant token output position, consistent with \parencite{arditi2024refusallanguagemodelsmediated}.
For SVD of a matrix \(M=U_M\Sigma_MV_M^\top\), we write
\(\mathcal V(M)=\mathrm{span}(V_M)\) to denote the right-singular subspace of \(M\).. We define stable rank as
\(\mathrm{sr}(M)=\|M\|_F^2/\|M\|_2^2\).

\subsection{Refusal mediation and existing attacks}

One of the well-known refusal ablation attacks involves estimating a refusal direction from a pair of harmless and harmful datasets \parencite{arditi2024refusallanguagemodelsmediated}.
For a fixed layer \(l\) let \(h_{harm,i},h_{benign,j}\in\RR^d\) denote the corresponding activations for harmful and benign prompts \(i\) and \(j\) evaluated at the first assistant token position as described above, where \(1\le i\le n_h\) and \(1\le j\le n_b\).
Define empirical means for those activations:
\begin{align}
\mu_{harm} \;=\; \frac{1}{n_h}\sum_{i=1}^{n_h} h_{{harm},i},
\qquad
\mu_{benign} \;=\; \frac{1}{n_b}\sum_{j=1}^{n_b} h_{{benign},j},
\end{align}
We operationalize the \emph{refusal direction} with the difference-in-means estimator:
\begin{align}
\rho \;=\; \mu_{harm} - \mu_{benign} \in \RR^d.
\label{eq:dom_refusal_dir_bg}
\end{align}
The \textit{refusal ablation attack} removes a projection on this normalized refusal vector from activations \(h\) at a run-time when a model is prompted, producing new activations $h^{(\mathrm{abl})}$:
\begin{align}
h^{(\mathrm{abl})}
\;\xleftarrow\;
h - \langle h^\top \cdot \hat{\rho} \rangle \hat{\rho}
\label{eq:refusal_ablation_bg}
\end{align}
In the above, $\hat{\rho} = \rho / \lVert\rho\rVert_2$ denotes normalized refusal direction. This refusal ablation attack jailbreaks the model, making it respond to harmful prompts it was trained to refuse.

Additionally, we define \textit{benign-centered refusal residual matrix} as:
\begin{align}
\Delta H
\;=\;
\begin{bmatrix}
(h_{{harm},1}-\mu_{benign})^\top\\
\vdots\\
(h_{{harm},n_h}-\mu_{benign})^\top
\end{bmatrix}
\in \mathbb{R}^{n_h\times d}
\end{align}
Given \Cref{eq:dom_refusal_dir_bg}, we can say that the matrix $\Delta H$ is a matrix of per-prompt refusal residuals $\rho_{i}$: its mean equals the difference-in-means refusal direction:
\begin{align}
\label{eq:residual_mean_equals_refusal_dir_bg}
\mu_\Delta
= \frac{1}{n_h}\sum_{i=1}^{n_h} (h_{{harm},i} - \mu_{benign})
= \mu_{harm} - \mu_{benign}
=
\rho
\in \RR^d
\qquad
\end{align}

Let
$ \Delta H_c = \Delta H-\mathbf 1_{n_h}\mu_\Delta^\top $ denote the centered refusal-residual matrix.
We operationalize the \emph{refusal subspace} as the affine subspace $\mu_\Delta+\mathcal V(\Delta H_c)$, where
\(\mathcal V(\Delta H_c)\) is its direction subspace.


\section{Refusal training gradients shape refusal residuals in transformer}
\label{sec:theory}

We analyze a transformer under supervised refusal training and show how repeated refusal-completion first tokens can concentrate activation gradients and, in turn, refusal residuals.
We conjecture that those refusal-completion gradients at refusal safety-training induce activation changes which can contribute to the refusal direction.
Limitations and assumptions are outlined in \Cref{appendix_limitations}.

\paragraph{Cross-entropy gradients to last layer activations.}
Consider the cross-entropy loss on the first token \(r\in V\) of a
refusal completion, represented by the one-hot target vector \(e_r\):
\begin{align}
\ell^{\mathrm{CE}}(e_r,p)
&=
-\log p(r)
=
-o_r+\log Z(o).
\end{align}
The activation gradient
\(\nabla_h\ell^{\mathrm{CE}}(e_r,p)\) defines the hypothetical
feature-space step
\begin{align}
\label{eq:virtual_update}
h^+
=
h-\gamma\nabla_h\ell^{\mathrm{CE}}(e_r,p),
\end{align}
We use this gradient, also termed as \emph{virtual update} by
\parencite{cha2026weightgrammatrixcaptures}, as a diagnostic to track feature trajectory.
For the last-layer activation \(h\), the gradient is
\begin{align}
\label{eq:ce_loss_h}
g_i
=
\nabla_h\ell_i^{\mathrm{CE}}(e_{r_i},p_i)
=
W(p_i-e_{r_i})
=
-w_{r_i}+Wp_i
\in\mathbb R^d.
\end{align}
Or, in a matrix form, for a set of $N$ examples with target refusals tokens \( r_i \), a one-hot matrix \(E_r=[e_{r_1},\ldots,e_{r_N}]^\top\) representing those refusal tokens, and a prediction matrix \(P=[p_1,\ldots,p_N]^\top\), the cross-entropy gradients to last-layer activations:
\begin{align}
\label{eq:last_layer_grads}
G_H^{\mathrm{CE}} =
(P - E_r) W^\top \in \mathbb{R}^{N \times d}.
\end{align}
The mean activation gradient is:
\begin{align}
\label{eq:mean_ce_loss_h}
\bar g
= \frac{1}{N}\sum_{i=1}^N g_i
= W(\bar p - \bar e_r) = - \bar w_r + W\bar p\in\RR^d.
\end{align}
In the \emph{static refusal} case with $r_i=r$ for all $i$ and $\bar e_r = e_r$, the mean direction becomes: $\bar g = - w_r + W\bar p $.





\paragraph{Shared direction across training example gradients.}
We can observe in \Cref{eq:ce_loss_h} and \Cref{eq:mean_ce_loss_h} that in case of static, repeating refusal token $r$, gradients across $N$ examples will be concentrated around shared direction $w_r$.
In case other term $W p$ does not cancel or bias $w_r$ too much (for example, when $p$ has large mass elsewhere), gradients concentrated around $w_r$ will be low rank.

\paragraph{From gradients to gradient-induced activation updates.}
Activations are changed after parameter updates (training steps) minimizing aforementioned loss: the raw gradient \(g_i\) is not itself the realized activation change.
Concretely, let \(J_i:=\nabla_\theta h_i(\theta)\) be its Jacobian.
Then, \(\nabla_\theta \ell_i^{\mathrm{CE}}=J_i^\top g_i\).
Under a small parameter update the realized change would be:
\begin{align}
\label{eq:first_order_hidden_state_change}
h_i(\theta-\eta\nabla_\theta \ell_i^{\mathrm{CE}})-h_i(\theta)
\approx
-\eta J_iJ_i^\top g_i .
\end{align}
We define the \emph{gradient-induced activation update} as the realized finite difference, taken with a negative sign for gradient-aligned comparison:
\begin{align}
\label{eq:gradient_induced_update_vector}
\mathfrak g_i
=
-
\left(
h_i(\theta-\eta\nabla_\theta \ell_i^{\mathrm{CE}})
-
h_i(\theta)
\right)
\approx
\eta J_iJ_i^\top g_i
\in\RR^d.
\end{align}
Stacking these updates across $N$ examples gives the \textit{gradient-induced activation update matrix} \(\mathcal G\):
\begin{align}
\label{eq:gradient_induced_update_matrix}
\mathcal G
=
\begin{bmatrix}
\mathfrak g_1^\top\\
\vdots\\
\mathfrak g_N^\top
\end{bmatrix}
\in\RR^{N\times d}.
\end{align}





\paragraph{Training-time origin of refusal directions}
Assuming gradient descent over harmful prompts paired with refusal completions, we hypothesize
that the refusal subspace $\mu_\Delta + \mathcal V(\Delta H_c)$ is approximated by the
sign-reversed affine subspace of gradient-induced activation updates $-\mu_{\mathcal G} + \mathcal V(\mathcal G_c)$, where $\mathcal G_c$ denotes a mean-centered matrix $\mathcal G$
\begin{align}
\label{eq:activation_gradient_affine_subspace}
\mu_\Delta + \mathcal V(\Delta H_c)
\approx
-\mu_{\mathcal G}+\mathcal V(\mathcal G_c).
\end{align}

Concretely, we test the above prediction using these three diagnostics:
\begin{align}
\label{eq:activation_gradient_mean_alignment}
\cos(\mu_\Delta,\mu_{\mathcal G})
&\approx -1,
&&\text{training steps move activations opposite to gradients,}
\\
\label{eq:activation_gradient_subspace_alignment}
\mathcal V(\Delta H_c)
&\approx \mathcal V(\mathcal G_c),
&&\text{right singular vectors span similar subspaces,}
\\
\label{eq:activation_gradient_stable_rank_alignment}
\operatorname{sr}\!\left(\Delta H\right)
&\approx
\operatorname{sr}\!\left(\mathcal G\right)
&&\text{their stable ranks are very close.}
\end{align}




\paragraph{Stable rank of first-token target matrix.}
Let \(E_r\in\RR^{N\times |V|}\) be the one-hot matrix of refusal first tokens \(r_1,\ldots,r_N\).
Across $N$ examples, for a token $k \in V$, let \(c_k = \#\{i:r_i=k\}\) denote its count , \(f_k = c_k/N\) denote its empirical frequency, and let \(m = |\{k:c_k>0\}|\) be the number of observed first-token buckets.
The following lemma, whose full  statement, proof, and a corollary appear in \Cref{app:Er_stable_rank}, shows that first-token concentration directly controls the stable rank of \(E_r\).
\begin{lemma}[Stable rank of refusal first-token targets]
\label{lem:Er_stable_rank}
The stable rank of \(E_r\) is
\begin{align}
\sr(E_r)
=
\frac{N}{\max_k c_k}
=
\frac{1}{\max_k f_k}.
\label{eq:Er_stable_rank_main}
\end{align}
\end{lemma}
In particular, \(\sr(E_r)=1\) when all examples share the same first refusal token.
For a fixed number \(m\) of observed first-token buckets, \(\sr(E_r)\) is maximized by balanced counts across the buckets
\begin{align}
\max \sr(E_r)
=
\frac{N}{\lceil N/m\rceil}.
\label{eq:Er_stable_rank_balanced_main}
\end{align}
Thus, a repeated refusal-completion first token, such as the token beginning “I'm sorry,” minimizes \(\operatorname{sr}(E_r)\), while balanced first-token frequencies increase it.
Since the last-layer first-token gradient matrix is \(G_H^{\mathrm{CE}}=(P-E_r)W^\top\), this gives a testable prediction: when the prediction term \(P\) and unembedding \(W\) do not collapse this target-rank signal, we can expect that increasing refusal first-token diversity can increase the stable rank of gradients \(G\) and their corresponding gradient-induced activation updates \(\mathcal G\).
We therefore report \(\operatorname{sr}(E_r)\), \(\operatorname{sr}(P-E_r)\), and \(\operatorname{sr}((P-E_r)W^\top)\) separately to measure $P$ and $W^\top$ preserve or distort the first-token target-matrix signal.


\paragraph{Stable rank distortion across layers: an idealized case.}
Assume in idealized case that for some corresponding Jacobian
$
J^{(l)} = 
\frac{\partial h^{(l)}}
     {\partial h^{(l-1)}}
\in \mathbb{R}^{d \times d}
$, gradients at previous layer $l-1$ satisfy
\begin{align}
G^{(l-1)} = G^{(l)}J^{(l)},
\label{eq:layerwise_gradient_matrix_recursion}
\end{align}
while that Jacobian is shared across training examples.
Assume also that the Jacobian is well-conditioned, i.e. its condition number $\kappa(J^{(l)}) = \|J^{(l)}\|_2\|(J^{(l)})^{-1}\|_2$ is small.
Under these assumptions, an inequality below bounds how much stable rank can change throughout remaining layers $1 \leq l \leq L-1$ under this layerwise map; see \Cref{app:sr_conditioned_factor} for the proof.

\begin{theorem}[Bounded stable rank distortion under a well-conditioned factor]
\label{thm:sr_well_conditioned}
Let \(A\in\RR^{N\times n}\), let \(B\in\RR^{n\times n}\) be invertible, and let \(\kappa(B) = \|B\|_2\|B^{-1}\|_2\) be a condition number.
Then
\begin{align}
\frac{1}{\kappa(B)^2}\sr(A)
\le
\sr(AB)
\le
\kappa(B)^2\sr(A).
\label{eq:sr_two_sided_condition_main}
\end{align}
\end{theorem}

Applying \Cref{thm:sr_well_conditioned} to \Cref{eq:layerwise_gradient_matrix_recursion} gives
\begin{align}
\frac{1}{\kappa(J^{(l)})^2}\sr(G^{(l)})
\le
\sr(G^{(l-1)})
\le
\kappa(J^{(l)})^2\sr(G^{(l)}).
\label{eq:sr_across_one_layer_main}
\end{align}
Iterating from layer \(L\) to layer \(l\) yields
\begin{align}
\frac{\sr(G^{(L)})}{\prod_{k=l+1}^{L}\kappa(J^{(k)})^2}
\le
\sr(G^{(l)})
\le
\left(\prod_{k=l+1}^{L}\kappa(J^{(k)})^2\right)\sr(G^{(L)}).
\label{eq:sr_across_many_layers_main}
\end{align}
Assuming in our case $G^{(L)} = G_H^{\mathrm{CE}} = (P-E_r)W^\top$ as per \Cref{eq:last_layer_grads}, we obtain the layerwise prediction for $l$-layer gradients $G^{(l)}$, bounded by functions of $G^{(L)}$ from both sides:
\begin{align}
\frac{\sr((P-E_r)W^\top)}{\prod_{k=l+1}^{L}\kappa(J^{(k)})^2}
\le
\sr(G^{(l)})
\le
\left( \prod_{k=l+1}^{L}\kappa(J^{(k)})^2 \right)\sr((P-E_r)W^\top).
\label{eq:sr_layerwise_from_last_token_main}
\end{align}
This gives another testable prediction: under our idealized conditions, stable rank increases of last- layer's gradients $G^{(L)}$ can also propagate back, increasing stable rank of layer-$l$ gradients $G^{(l)}$.
The next parts refine the idealized condition above by introducing a data-dependent variant of a shared factor and data-dependent stable rank transfer factors in order to construct tighter bounds.


\begin{definition}[Stable rank transfer factor]
For nonzero gradient matrices $G^{(L)}$ and $G^{(l)}$ computed for the same $N$ examples at layers $L$ and $l$ respectively, their stable rank transfer factor is the ratio of their stable ranks:
\begin{align}
\tau^{L\to l}
:=
\frac{
\operatorname{sr}\left(G^{(l)}\right)
}{
\operatorname{sr}\left(G^{(L)}\right)
}.
\end{align}
Thus, $\tau^{L\to l}>1$ denotes stable rank amplification, while $\tau^{L\to l}<1$ denotes stable rank contraction between the selected layers. We similarly define transfer factors for gradient-induced updates $\mathcal G$.
\end{definition}

To refine impractical assumptions of \Cref{eq:layerwise_gradient_matrix_recursion}, instead of relying on an idealized, shared Jacobian, we introduce an appropriate approximation, or a shared cross-layer relation map.
\begin{definition}[Top-$k$ source-subspace shared map]
For nonzero gradient matrices $G^{(L)}$ and $G^{(l)}$ computed for the same $N$ examples at layers $L$ and $l$ respectively, a top-$k$ source-subspace shared map $\widehat J_k^{L\to l}$ is a shared data-dependent low-rank linear approximation for the observed cross-layer relation between gradients fitted on $N$ examples and restricted to a top-$k$ principal subspace of $G^{(L)}$. It is defined as follows: 
\begin{align}
\widehat J_k^{L\to l} = V_k^{(L)} \Gamma_k^{L\to l}, \\
\widehat G_k^{(l)} = G^{(L)}\widehat J_k^{L\to l}.
\end{align}
where matrix $V_k^{(L)} \in \mathbb R^{d\times k}$ denotes top $k$ right singular vectors of $G^{(L)}$, therefore $\operatorname{rank}(\widehat J_k^{L\to l}) \leq k$.
$\Gamma_k^{L\to l} \in \mathbb R^{k\times d}$ is a coefficient matrix estimated from $N$ gradients.
$\widehat G_k^{(l)}$ is the layer-$l$ gradient matrix implied by the top-$k$ source-subspace shared map $\widehat J_k^{L\to l}$.

For a nonzero $\widehat G_k^{(l)}$, the map-implied stable rank transfer factor is the stable rank transfer factor our shared map implies on $G^{(L)}$: 
\begin{align}
\widehat\tau_k^{L\to l}=
\frac{
\operatorname{sr}\left(\widehat G_k^{(l)}\right)
}{
\operatorname{sr}\left(G^{(L)}\right)
}
\end{align}
\end{definition}

Next, we account of errors resulting from the above approximation.
\begin{definition}[Relative transfer factor fit error]
For the observed $\tau^{L\to l}$ stable rank transfer factor, a shared top-$k$ source-subspace map $\widehat J_k^{L\to l}$, and a stable rank transfer factor $\widehat\tau_k^{L\to l}$ implied by it, define a relative transfer factor error of $\widehat J_k^{L\to l}$ as follows:
\begin{align}
\epsilon_{\tau,k}^{L\to l}
=
\frac{
\left|
\widehat\tau_k^{L\to l} -
\tau^{L\to l}
\right|
}{
\tau^{L\to l}
}.
\end{align}
\end{definition}

\Cref{thm:sr_well_conditioned} gives a worst-case condition-number bound under an idealized shared linear factor. Lemma~\ref{lem:stable-rank-transfer-fit-bound} instead shows how an upper bound on the transfer-factor fit error yields a data-dependent two-sided bound on the stable rank of $G^{(l)}$ (see Lemma~\ref{lem:stable-rank-transfer-fit-bound-proof} for the proof); we will later estimate that upper bound empirically.

\begin{lemma}[Stable rank bounds from transfer factor fit error]
\label{lem:stable-rank-transfer-fit-bound}

Let \(G^{(L)},G^{(l)}\in\mathbb{R}^{N\times d}\) be nonzero gradient
matrices, and let \(\widehat{J}_k^{L\to l}\) be a top-$k$ source-subspace
shared map with fitted output
\[
\widehat{G}_k^{(l)}
=
G^{(L)}\widehat{J}_k^{L\to l}.
\]
Suppose that the relative transfer-factor fit error is bounded by some $\epsilon_{max}$:
\[
\epsilon_{\tau,k}^{L\to l} \leq \epsilon_{max} <1.
\]
Then
\[
\frac{
    \widehat{\tau}_k^{L\to l}
}{
    1+\epsilon_{max}
}
\operatorname{sr}\!\left(G^{(L)}\right)
\leq
\operatorname{sr}\!\left(G^{(l)}\right)
\leq
\frac{
    \widehat{\tau}_k^{L\to l}
}{
    1-\epsilon_{max}
}
\operatorname{sr}\!\left(G^{(L)}\right).
\]
\end{lemma}

When $\epsilon_{max}$ is obtained by measuring the fit error on the
same \(N\) examples used to estimate \(\widehat J_k^{L\to l}\),
Lemma~\ref{lem:stable-rank-transfer-fit-bound} is an in-sample fit
certificate rather than an independent prediction of
\(\operatorname{sr}(G^{(l)})\). It becomes a predictive propagation
bound only when an independently established error bound
\(\epsilon_{\max}\) is assumed to continue holding for additional source
gradient matrices. Moreover, agreement at the level of stable rank
does not imply matrix-level agreement; we evaluate the latter
separately using relative Frobenius error, relative spectral error,
and row-wise cosine similarity.






\section{Empirical validation}
\label{sec:empirical}

\subsection{Experimental setup}

\paragraph{Target models.}
We perform measurements against \texttt{allenai/OLMo-2-0425-1B-Instruct}, an instruction-tuned OLMo 2 1B model with 16 transformer blocks, 16 attention heads, and hidden dimension \(d=2048\).
For the training-stage analysis in \Cref{subsec:training_stage_emergence}, we evaluate its additional checkpoints: 
\texttt{allenai/OLMo-2-0425-1B} (\texttt{olmo\_base}),
\texttt{allenai/OLMo-2-0425-1B-SFT} (\texttt{olmo\_sft}),
\texttt{allenai/OLMo-2-0425-1B-DPO} (\texttt{olmo\_dpo}),
\texttt{allenai/OLMo-2-0425-1B-RLVR1} (\texttt{olmo\_rlvr}), and
\texttt{allenai/OLMo-2-0425-1B-Instruct} (\texttt{olmo\_inst}) \parencite{allenai2025olmo204251binstruct}.

\paragraph{Target datasets.}
We use AdvBench (\texttt{adv}),
WildJailbreak evaluation and training splits (\texttt{wld}, \texttt{wld\_tr}),
and XSTest-Response subsets (\texttt{xst\_hr}, \texttt{xst\_hr\_jb}, \texttt{xst\_rf}, \texttt{xst\_rf\_jb})
as target prompt sets.
Our frozen-model ablation analyses on these datasets use layer \(7\) for refusal vector estimation.
For controlled fine-tuning, we use CAMEL chemistry dataset \textit{camel-ai/chemistry} \parencite{li2023camel} by pairing its prompts with refusals sampled from a predefined set of refusals we manually collect; we use layer \(8\) to estimate refusal vector.

\paragraph{Metrics.} We use refusal (harm) rate, or a proportion of refused (or harmful) prompts to all prompts as downstream metrics, using \texttt{allenai/wildguard} as a judge LLM for refusal (harm) scores; refusal delta (change under attack) is a difference between pre-intervention and post-intervention refusal rates.
To test the affine-subspace prediction in \Cref{eq:activation_gradient_affine_subspace}, we compare three quantities for a benign-centered residual matrix \(\Delta H\) paired with gradient-induced activation updates \(\mathcal G\) (or, similarly, paired with raw gradients \(G\)): (1) cosine similarity between their \(\Delta H\) and \(\mathcal G\) mean vectors;
(2) stable rank on the uncentered $\Delta H$ and $\mathcal G$;
(3) average principal-angle overlap between their centered top-\(k\) right-singular subspaces.
We report stable ranks and transfer-factor summary statistics across
conditions. Shared-map fits are evaluated using source-subspace energy,
relative Frobenius and spectral errors, row-wise cosine similarity, and
relative transfer-factor fit error, with identity and scalar-identity
baselines.
Full definitions appear in \Cref{app:matrix_metrics}.

\subsection{Refusal direction across released post-training checkpoints}
\label{subsec:training_stage_emergence}

\begin{figure}[h]
    \centering
    \begin{subfigure}[h]{0.31\linewidth}
        \centering
        \includegraphics[
            width=\linewidth,
        ]{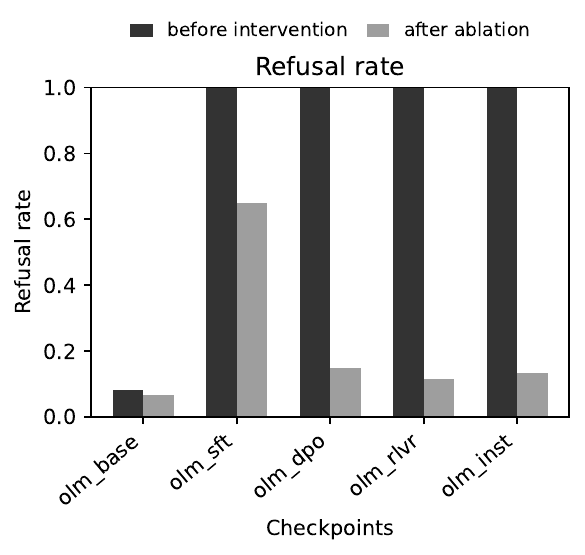}
        \label{fig:_profiler_to_upload2____refusal_decomp___layer15:a}
    \end{subfigure}
    \begin{subfigure}[h]{0.32\linewidth}
        \centering
        \includegraphics[
            width=\linewidth,
        ]{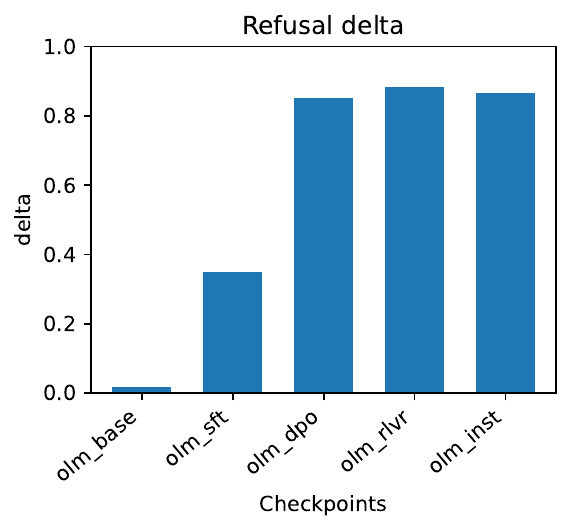}
        \label{fig:_profiler_to_upload2____refusal_decomp___layer15:e}
    \end{subfigure}
    \caption{Descriptive comparison of released OLMo checkpoints along
    Base \(\to\) SFT \(\to\) DPO \(\to\) RLVR1 \(\to\) Instruct.
    The base checkpoint is not attack-effective under this estimator.
    A functional refusal direction is present by SFT, while post-SFT checkpoints exhibit larger refusal deltas than SFT.}
    \label{fig:_profiler_to_upload2____refusal_decomp___layer15}
\end{figure}

We first ask when the difference-in-means vector becomes a functional refusal direction.
Using 60 AdvBench prompts, we estimate refusal vectors $\rho$ at each layer for the OLMo checkpoints and measure refusal rates before/after ablation, refusal vector norm, cosine similarity to the final \texttt{olmo\_inst} direction, and stable rank of \(\Delta H\).
\Cref{fig:_profiler_to_upload2____refusal_decomp___layer15} shows that the base checkpoint has little refusal behavior and a negligible refusal delta under this estimator, whereas all four post-base checkpoints refuse nearly all selected prompts before intervention.
A functional refusal-mediating direction is present at the SFT checkpoint, but its refusal delta is substantially smaller than those of the DPO, RLVR1, and Instruct checkpoints.
Because the checkpoints differ in their objectives, data, and optimization histories, this comparison does not isolate the causal effect of any individual post-training stage.
\Cref{fig:_profiler_to_upload2____refusal_decomp___layer15__pairwise} shows that the refusal vectors estimated at SFT and the later checkpoints remain highly aligned with the final \texttt{olmo\_inst} direction across layers, while their stable rank and vector norm profiles largely overlap.
Thus, the released post-SFT checkpoints approximately preserve the refusal geometry even though SFT is less attack-effective in our settings.
Finally, \Cref{fig:_profiler_to_upload2_additional_layer} shows that, for \texttt{olmo\_inst}, directions estimated at layers 7--10 produce the largest refusal deltas.

\subsection{Gradient-induced activation updates match benign-centered refusal residuals}
\label{subsec:gradient_update_residual_match}

We next test the affine-subspace prediction in \Cref{eq:activation_gradient_affine_subspace}: benign-centered refusal residuals \(\Delta H\) should share their mean direction, principal subspace, and stable rank with gradient-induced activation updates \(\mathcal G\), up to the sign induced by gradient descent.
Using the target datasets introduced above, we compute \(\Delta H\) and \(\mathcal G\) for three target-completion types: (1) model-generated refusal completions, (2) post-ablation completions, and (3) fixed pseudorandom English-token completions.
For each dataset and output type, we compare \(\Delta H\) and \(\mathcal G\) using mean-direction cosine similarity, average principal-angle overlap between centered right-singular subspaces, and stable rank.
We make similar comparisons for \(\Delta H\) and raw gradients \(G\) as well.

\begin{figure}[h]
    \centering
    \begin{minipage}{\linewidth}
        \centering

        \includegraphics[width=0.42\linewidth]{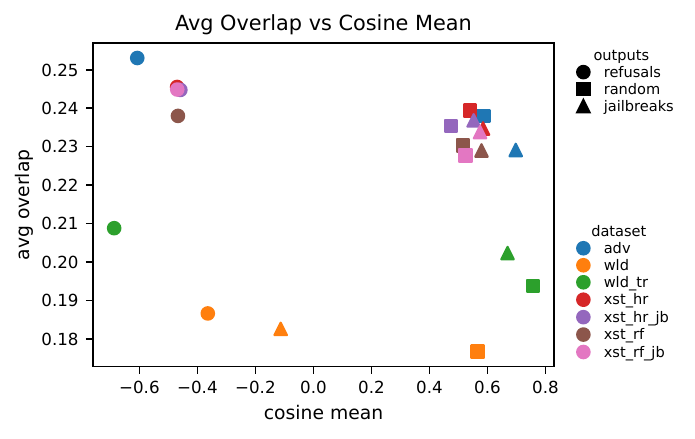}
        \hfill
        \includegraphics[width=0.42\linewidth]{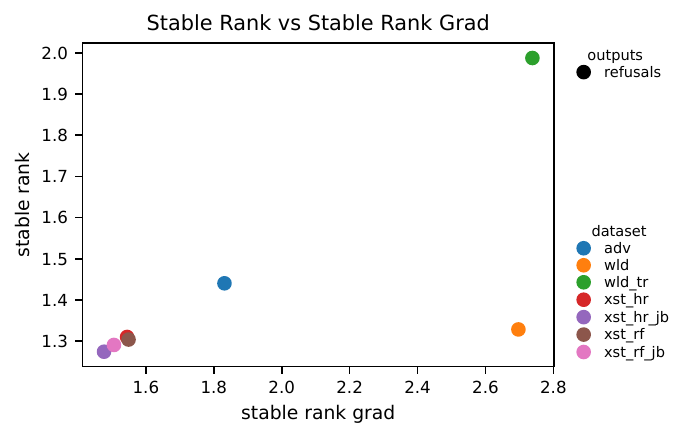}

        \vspace{0.25em}

        \includegraphics[width=0.42\linewidth]{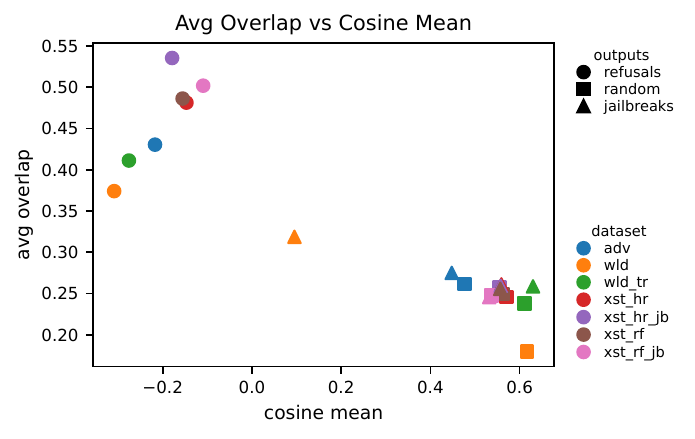}
        \hfill
        \includegraphics[width=0.42\linewidth]{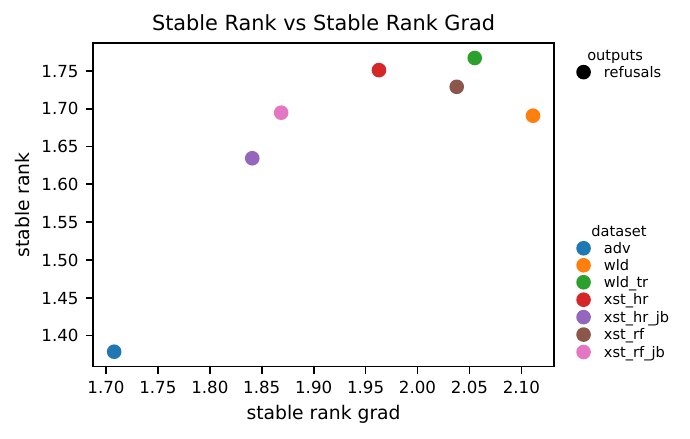}

        \caption{Gradient-induced activation updates explain benign-centered refusal residuals. Left column: mean alignment and average principal-angle overlap. Right column: stable rank comparison. Top row: layer \(7\); bottom row: layer \(15\).}
        \label{fig:gradient_update_residual_metrics}
    \end{minipage}
\end{figure}

The left column of \Cref{fig:gradient_update_residual_metrics} shows that refusal-based means are negatively aligned with the difference-in-means refusal vector, while ablated-jailbreak-based and random-English-based means are positively aligned with it.
Thus, refusal and non-refusal targets move along approximately the same refusal axis but in opposite directions, supporting \Cref{eq:activation_gradient_mean_alignment}.
This effect appears primarily for the refusal directions estimated at middle layer \(7\) and partially from the last layer \(15\).
The same column also shows average principal-angle overlap between \(\mathcal V(\Delta H_c)\) and \(\mathcal V(\mathcal G_c)\).
At layer \(15\), refusal-target updates subspace overlap refusal subspace more than both
random-English and ablated-jailbreak controls gradient update subspaces for every evaluated dataset, indicating stronger last-layer agreement between direction subspaces of refusal gradient updates \(\mathcal V(\mathcal G_c)\) and refusal subspace \(\mathcal V(\Delta H_c)\).
At layer \(7\), overall subspace overlap is lower, and the overlap appears to be more strongly associated with
the dataset than with the completion type: the three target types within a dataset are relatively close, while the evaluation split of WildJailbreak (\texttt{wld}) has lower overlap than its training split (\texttt{wld\_tr}) and most other datasets.

The right column of \Cref{fig:gradient_update_residual_metrics} shows that, for refusal targets, the stable ranks of \(\mathcal G\) and \(\Delta H\) are close across datasets for both layer \(7\) and layer \(15\).
The observed stable ranks are below the sample size \(N=60\), around \(1\)--\(3\), indicating genuine spectral concentration and excluding trivial rank cap from insufficient datapoints.
We observe the same qualitative relationship across most transformer layers.

Together, these three diagnostics support the main empirical
implication of \Cref{eq:activation_gradient_affine_subspace}: signed mean-direction
alignment, partial and layer-dependent principal-subspace overlap, and
closely matched stable ranks.
They support an approximate affine-subspace relationship.

As a complementary uncentered diagnostic, at the refusal-estimating
layer \(7\), the leading right-singular axis of \(\mathcal G\) has
mean absolute cosine similarity \(0.56\) with the refusal direction
and \(0.96\) with the mean of \(\mathcal G\) across the seven evaluated
datasets.
Thus, the dominant uncentered gradient-update axis of \(\mathcal G\) is nearly collinear, up to
sign, with the mean gradient-update vector of \(\mathcal G\) and moderately
aligned with the refusal direction.
Together with the low stable rank of \(\mathcal G\), this makes the
mean-direction and leading-axis descriptions of its update geometry
nearly equivalent in this setting and motivates a direct comparison
between difference-in-means refusal ablation and ablation along the
top largest principal directions.


\subsection{Refusal first-token diversity increases gradient stable ranks}
\label{subsec:first_token_diversity_gradient_rank_end_to_end}


We next test on a frozen model whether varying refusal first tokens on a fixed prompt set changes the stable ranks of first-token gradients from the last layer to refusal-mediating middle layers.
Using the WildJailbreak training split, we fix 80 harmful prompts whose model-generated refusals begin with the common first token ``I''.
For each \(m\in\{1,2,4,6,8,10,12,14,16\}\), we pair each prompt with a refusal sampled independently with replacement from a pool of \(m\) short phrases with distinct first tokens.
For each setting, we compute \(E_r\), \(P\), and \(P-E_r\), the last-layer activation gradient matrix \((P-E_r)W^\top\), and the layerwise  gradient matrices \(G^{(l)}\) for the first-token cross-entropy loss.
We use raw gradients in this experiment to test the cross-layer rank-transfer analysis from \Cref{sec:theory}; gradient-induced activation updates are considered in the next subsection.

\begin{figure}[h]
    \centering
    \begin{subfigure}[b]{0.31\linewidth}
        \centering
        \includegraphics[width=\linewidth]{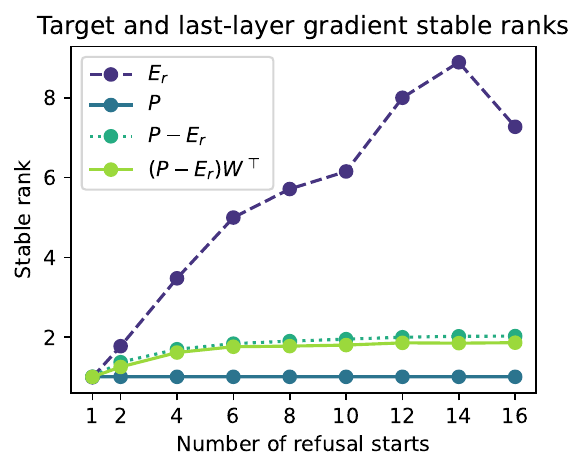}
        \label{fig:4_4_ranks:head}
    \end{subfigure}
    \begin{subfigure}[b]{0.34\linewidth}
        \centering
        \includegraphics[width=\linewidth]{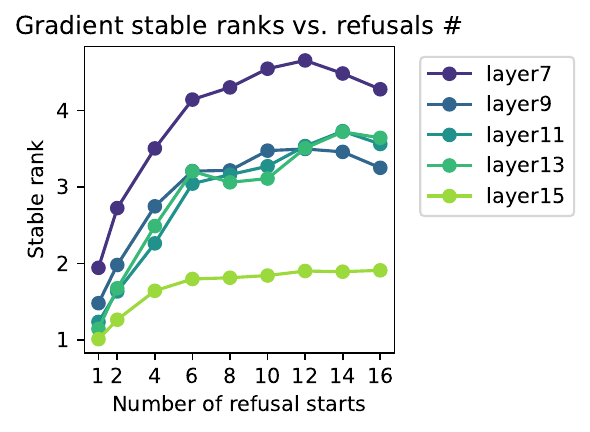}
        \label{fig:4_4_ranks:layers}
    \end{subfigure}
    \begin{subfigure}[b]{0.32\linewidth}
        \centering
        \includegraphics[width=\linewidth]{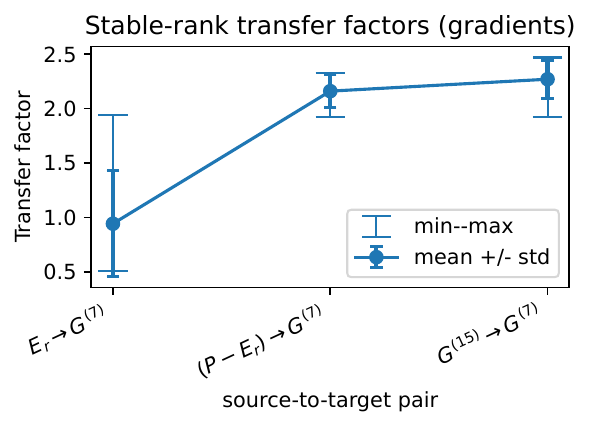}
        \label{fig:4_4_ranks:transfer}
    \end{subfigure}
    \caption{
    Stable ranks and transfer under refusal-start diversification.
    Left: stable ranks of the target-space quantities and the resulting last-layer gradient.
    Middle: stable ranks of the first-token gradients at selected middle-to-late layers.
    Right: across-setting layerwise stable rank transfer factor mean$\pm std$ and min/max.
    The stable ranks of \(E_r\), \(P-E_r\), and \(G^{(l)}\) increase overall.
    The \(G^{(15)}\!\to G^{(7)}\) transfer factor varies substantially less than the end-to-end \(E_r\!\to G^{(7)}\) ratio.
    }
    \label{fig:4_4_ranks}
\end{figure}

\Cref{fig:4_4_ranks} shows that all target-space quantities and gradients across layers 7-15 increase with the number of refusals: gradient stable rank increase is bounded and quickly saturates for the last layer due to constant $P$, but amplified throughout mid-layers on a backward pass.
Across the nine settings, the observed layerwise transfer factor mean \(\pm\) sample standard deviation are
\(\tau^{15\to7}=2.270\pm0.175\), with range \([1.925,2.469]\).
Thus, every evaluated setting exhibits approximately twofold stable-rank amplification from layer 15 to layer 7.
By contrast, the end-to-end \(E_r\!\to G^{(7)}\) ratio is more variable, with mean \(0.942\pm0.486\) and range \([0.504,1.942]\).

For the \(m=16\) setting, we additionally fit the top-\(k\) source-subspace shared map from \(G^{(15)}\) to \(G^{(7)}\) on the same 80 examples and sweep \(k\).
At \(k=20\), the selected source subspace captures \(99.8\%\) of the source-gradient energy; the relative Frobenius and spectral errors are \(0.594\) and \(0.316\), the mean row-wise cosine similarity is \(0.806\), and the relative transfer-factor fit error is \(0.343\).
The fitted map yields lower reconstruction and transfer-factor errors and higher row-wise cosine similarity than the identity and best scalar-identity baselines.
This provides evidence that the observed in-sample cross-layer relation admits a meaningful but incomplete shared linear approximation.
The complete rank sweep and reduced-map spectrum are reported in \Cref{fig:4_4_shared_map,fig:4_4_shared_map_spectrum}.

\subsection{Refusal first-token diversity increases residual rank and weakens single-vector ablation}
\label{subsec:first_token_diversity_rank}

We next test on a frozen model whether increasing the diversity of first tokens across refusal phrases (through varying selection of prompts in the dataset) increases the effective rank of benign-centered refusal residuals and gradient-induced activations updates.
Using the WildJailbreak training split, we select harmful prompts that the target model refuses, keeping the baseline refusal rate fixed across subsets.
We construct nine 80-example subsets drawing from \(m\in\{1,2,4,6,8,10,12,14,16\}\) refusal first-token buckets.
For each subset, we compute benign-centered refusal residuals \(\Delta H\), raw first-token activation gradients \(G^{(l)}\), gradient-induced activation updates \(\mathcal G^{(l)}\), and refusal and harmfulness scores before and after difference-in-means ablation.
Both \(G^{(l)}\) and \(\mathcal G^{(l)}\) are computed using the original WildJailbreak response paired with each selected prompt.
We then plot the ablation score deltas against the stable rank of \(\Delta H\).
We select layer \(7\) as the earliest layer in the range of effective layers.

\begin{figure}[h]
    \centering
    \begin{subfigure}[b]{0.3\linewidth}
        \centering
        \includegraphics[width=\linewidth]{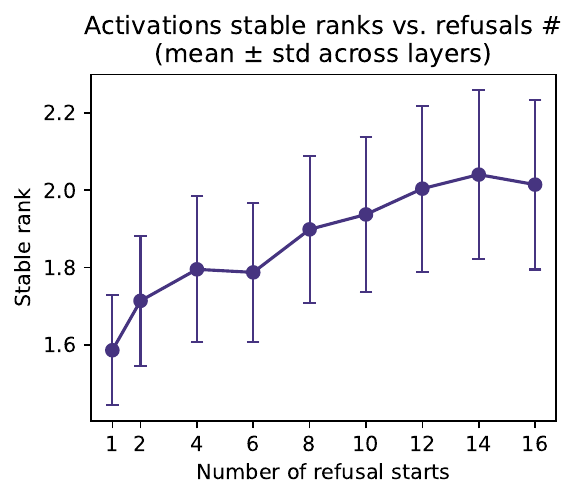}
        \label{fig:____stable_rank_vs_x__layerwise_overlay:a}
    \end{subfigure}
    \begin{subfigure}[b]{0.34\linewidth}
        \centering
        \includegraphics[width=\linewidth]{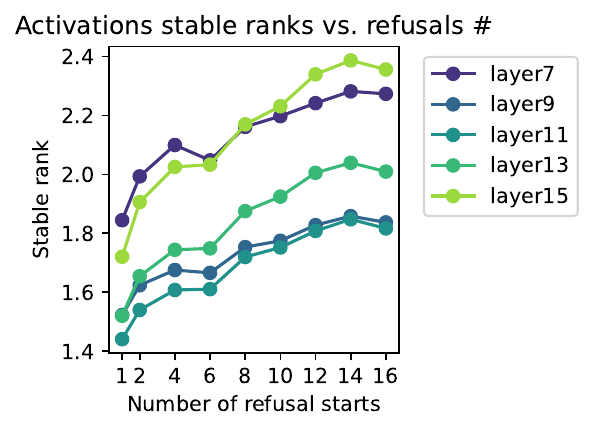}
        \label{fig:____stable_rank_vs_x__layerwise_overlay:aa}
    \end{subfigure}
    \begin{subfigure}[b]{0.33\linewidth}
        \centering
        \includegraphics[width=\linewidth]{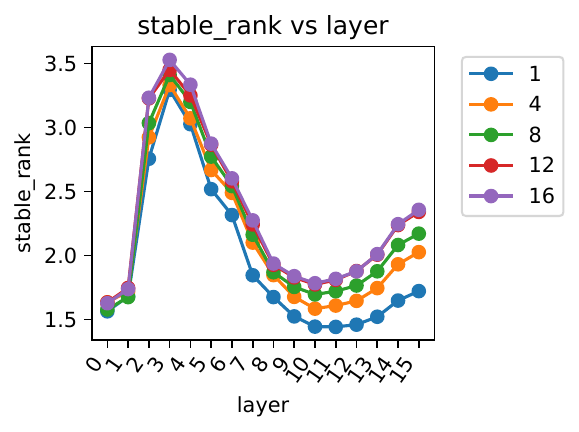}
        \label{fig:____stable_rank_vs_x__layerwise_overlay:b}
    \end{subfigure}
    \caption{Stable rank of benign-centered refusal residuals \(\Delta H\). Left plot averaged ranks over layers 7-15. Increasing refusal first-token diversity raises \(\operatorname{sr}(\Delta H)\) in middle-to-late layers.}
    \label{fig:____stable_rank_vs_x__layerwise_overlay}
\end{figure}

\Cref{fig:____stable_rank_vs_x__layerwise_overlay} shows that increasing the number of refusal first-token buckets raises \(\Delta H\) stable rank.
The effect is visible across middle-to-late layers, while the residual representation at earlier layers retain their stable ranks.
\Cref{fig:____stable_rank_grad_vs_x__layerwise_overlay} shows similar trend for gradient-induced activation updates, but with saturation at the end.
As the refusal first tokens become more diverse, the stable rank of \(\mathcal G\) also increases across most layers, however it saturates at 12-16 refusal starts.

\begin{figure}[h]
    \centering
    \begin{subfigure}[b]{0.32\linewidth}
        \centering
        \includegraphics[width=\linewidth]{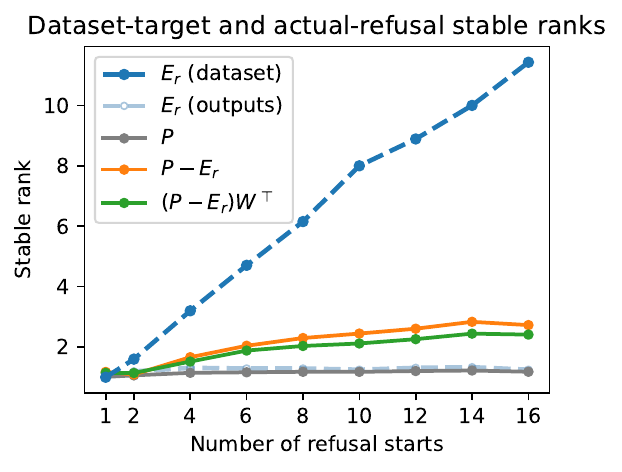}
        \label{fig:4_5_rank_propagation:head}
    \end{subfigure}
    \begin{subfigure}[b]{0.34\linewidth}
        \centering
        \includegraphics[width=\linewidth]{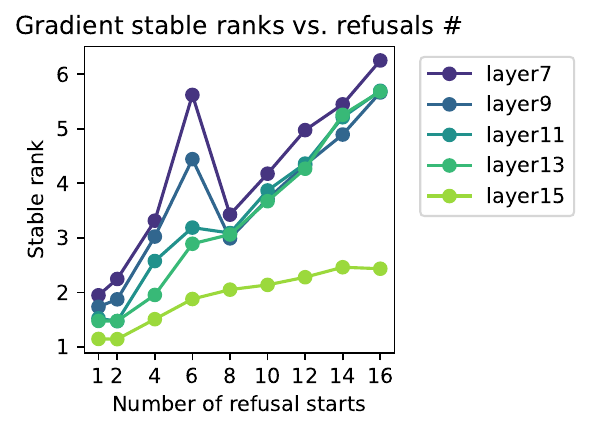}
        \label{fig:4_5_rank_propagation:layers_raw}
    \end{subfigure}
    \begin{subfigure}[b]{0.32\linewidth}
        \centering
        \includegraphics[width=\linewidth]{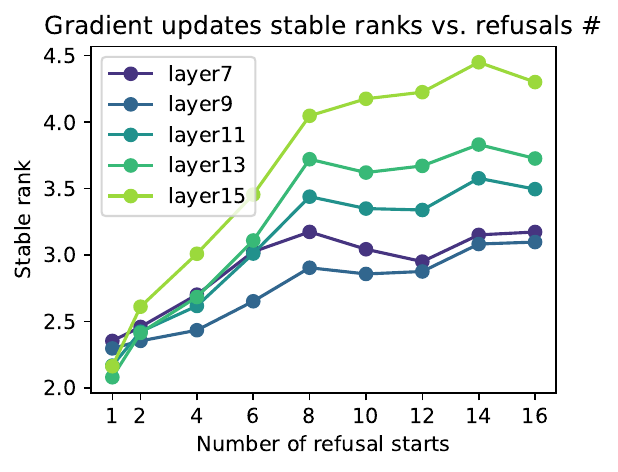}
        \label{fig:4_5_rank_propagation:layers}
    \end{subfigure}
    \caption{
    Stable rank propagation under refusal-start diversification by prompt selection.
    Left: stable ranks of the first-token target matrix \(E_r\), prediction matrix \(P\), logit-gradient matrix \(P-E_r\), and last-layer activation-gradient matrix \((P-E_r)W^\top\).
    Middle: stable ranks of the first-token gradients \(G^{(l)}\).
    Right: stable ranks of the corresponding gradient-induced activation updates \(\mathcal G^{(l)}\).
    The target-space and layerwise stable ranks increase overall as refusal-start support broadens.
    }
    \label{fig:4_5_rank_propagation}
\end{figure}

\Cref{fig:4_5_rank_propagation} shows the same overall rank increase from the dataset-target quantities through the raw gradients \(G^{(l)}\) and the gradient-induced activation updates \(\mathcal G^{(l)}\). Notably, the dataset-derived \(E_r\) rank rises sharply across the prompt subsets, whereas the model’s actual output-refusal-start \(E_r\) rank remains much lower. Thus, the dataset responses contain substantially more first-token diversity than the frozen model realizes in its own completions.

\begin{figure}[h]
    \centering
    \begin{subfigure}[b]{0.35\linewidth}
        \centering
        \includegraphics[width=\linewidth]{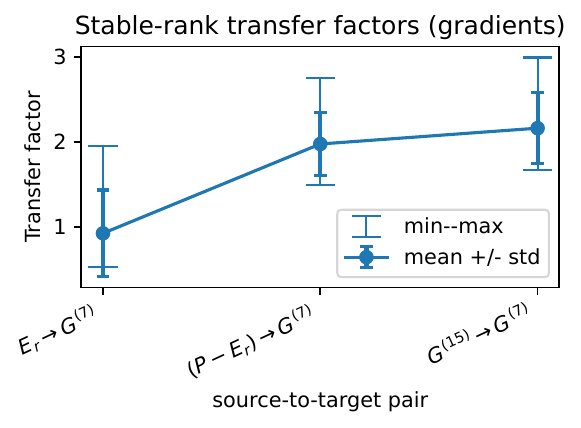}
        \label{fig:4_5_transfer_factors:transfer_raw}
    \end{subfigure}
    \begin{subfigure}[b]{0.37\linewidth}
        \centering
        \includegraphics[width=\linewidth]{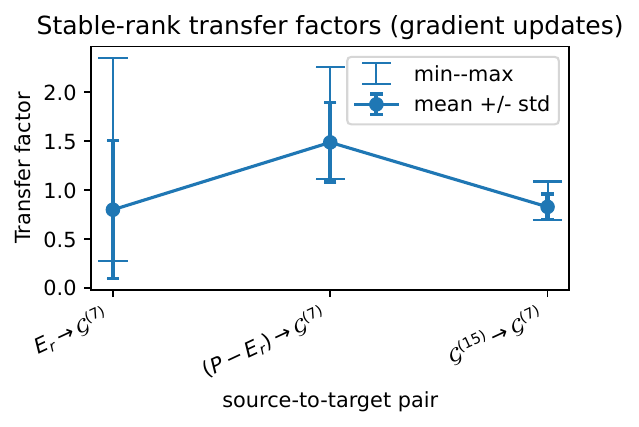}
        \label{fig:4_5_transfer_factors:transfer}
    \end{subfigure}
    \caption{
    Variation of stable rank transfer factor for raw gradients and gradient-induced activation updates across experimental settings.
    Left: raw gradient transfer factor.
    Right: gradient-induced activation deltas transfer factor.
    Raw gradients consistently amplify stable rank between layers 15 and 7, whereas the gradient update matrices undergo mild stable-rank contraction on average.
    }
    \label{fig:4_5_transfer_factors}
\end{figure}

Across the nine settings, raw gradient transfer factor is \(\tau_G^{15\to7}=2.161\pm0.418\), reported as mean \(\pm\) sample standard deviation, with range \([1.671,2.993]\).
Every evaluated setting therefore exhibits stable rank amplification from \(G^{(15)}\) to \(G^{(7)}\), by approximately a factor of two on average, although the amplification magnitude varies across subsets.
For gradient-induced activation updates, the corresponding factor is
\(\tau_{\mathcal G}^{15\to7}=0.829\pm0.132\), with range \([0.698,1.088]\) indicating mild stable-rank contraction on average.
The layer-\(15\)-to-layer-\(7\) update transfer is more stable across subsets than the full target-to-middle-layer ratios.
The complete in-sample shared-map rank sweeps and reduced-map spectra for \(G\) and \(\mathcal G\) are reported in \Cref{fig:4_5_shared_map_gradients,fig:4_5_shared_map_gradients_spectrum,fig:4_5_shared_map_updates,fig:4_5_shared_map_updates_spectrum}.

We then test whether higher stable rank of \(\Delta H\) corresponds to weaker single-vector ablation.
\Cref{fig:__stable_rank_vs_harm_delta_layer7} reports harmfulness scores before and after ablation, and plots the harmfulness delta against \(\operatorname{sr}(\Delta H)\).
Subsets with higher-rank \(\Delta H\) have smaller harmfulness deltas, suggesting that a less one-dimensional residual subspace is harder to suppress with one vector.
\Cref{fig:__stable_rank_vs_refusal_delta_layer7} shows, similarly, that the refusal delta slightly decreases as \(\operatorname{sr}(\Delta H)\) increases, indicating that higher-rank residuals \(\Delta H\) weaken the effect of single-vector refusal ablation.

\begin{figure}[h]
    \centering
    \begin{subfigure}[b]{0.3\linewidth}
        \centering
        \includegraphics[width=\linewidth]{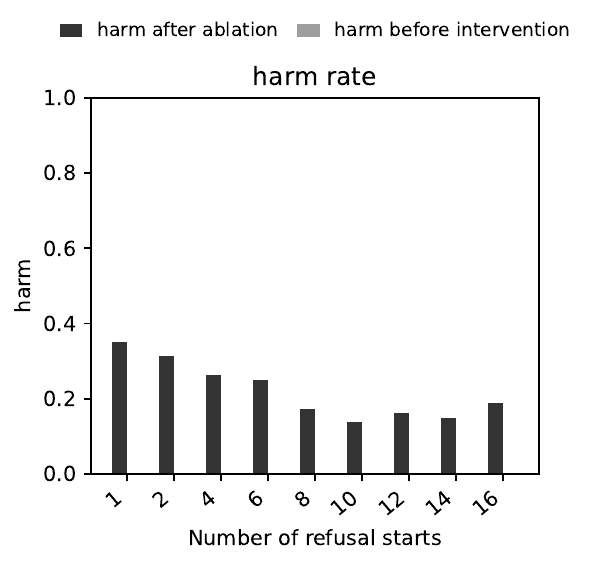}
        \label{fig:__stable_rank_vs_harm_delta_layer7:a}
    \end{subfigure}
    \begin{subfigure}[b]{0.3\linewidth}
        \centering
        \includegraphics[width=\linewidth]{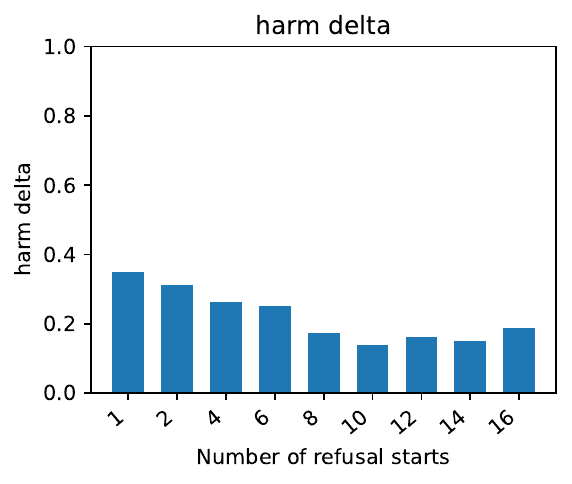}
        \label{fig:__stable_rank_vs_harm_delta_layer7:aa}
    \end{subfigure}
    \begin{subfigure}[b]{0.37\linewidth}
        \centering
        \includegraphics[width=\linewidth]{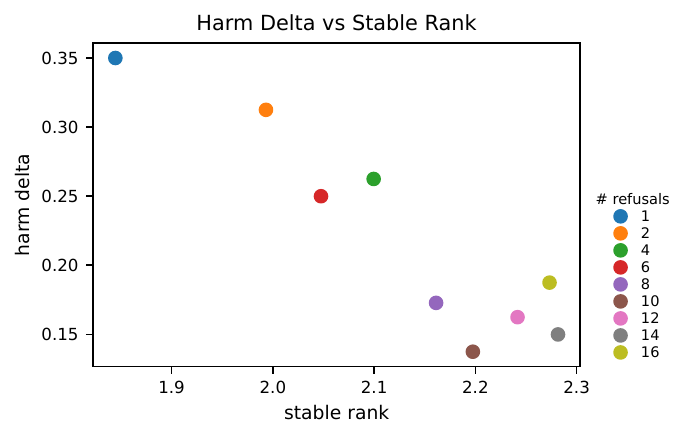}
        \label{fig:__stable_rank_vs_harm_delta_layer7:b}
    \end{subfigure}
    \caption{Harmfulness score decomposition and harmfulness delta versus \(\operatorname{sr}(\Delta H)\). Higher refusal residuals stable rank is associated with smaller harmfulness deltas under difference-in-means ablation.}
    \label{fig:__stable_rank_vs_harm_delta_layer7}
\end{figure}

\begin{figure}[h]
    \centering
    \begin{subfigure}[b]{0.3\linewidth}
        \centering
        \includegraphics[width=\linewidth]{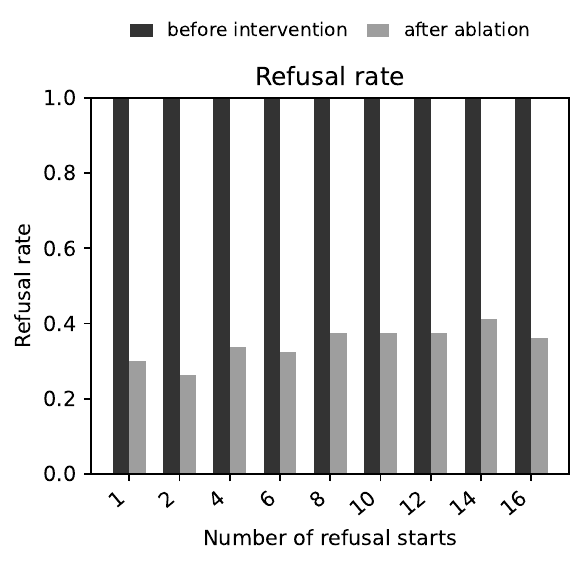}
        \label{fig:__stable_rank_vs_refusal_delta_layer7:a}
    \end{subfigure}
    \begin{subfigure}[b]{0.3\linewidth}
        \centering
        \includegraphics[width=\linewidth]{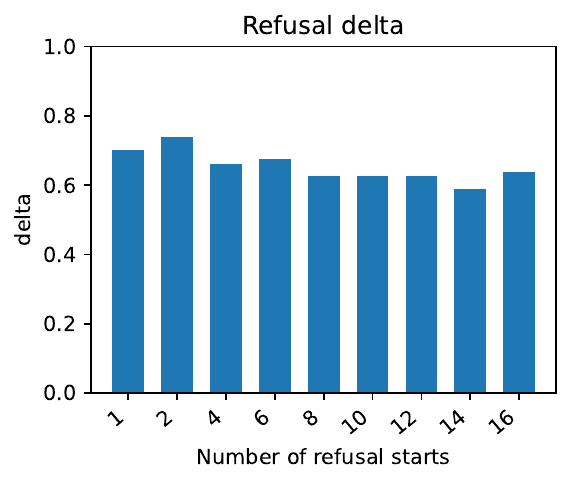}
        \label{fig:__stable_rank_vs_refusal_delta_layer7:aa}
    \end{subfigure}
    \begin{subfigure}[b]{0.37\linewidth}
        \centering
        \includegraphics[width=\linewidth]{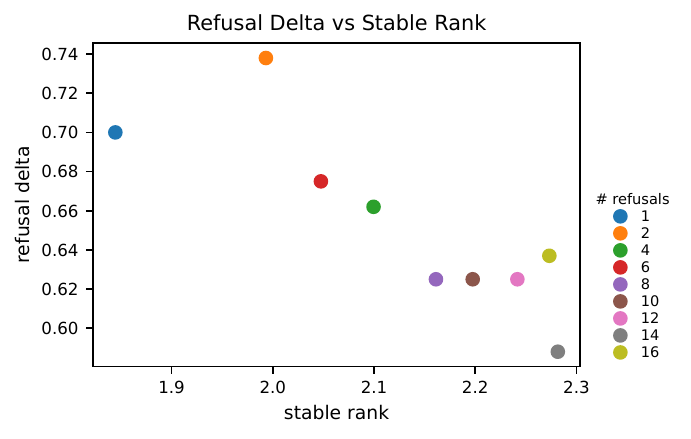}
        \label{fig:__stable_rank_vs_refusal_delta_layer7:b}
    \end{subfigure}
    \caption{Refusal score decomposition and refusal delta versus \(\operatorname{sr}(\Delta H)\). Higher activation stable rank is associated with smaller refusal deltas, consistent with weaker single-vector ablation.}
    \label{fig:__stable_rank_vs_refusal_delta_layer7}
\end{figure}

Overall, this fixed-model subset analysis supports the rank mechanism: increasing refusal first-token diversity raises the effective rank of both \(\Delta H\) and \(\mathcal G\), while higher activation stable rank is associated with smaller refusal and harmfulness deltas under single-vector ablation.

\subsection{Diverse-refusal fine-tuning raises refusal residuals rank and weakens ablation}
\label{subsec:chem_ft_rank}

We next test whether the rank--vulnerability pattern from \Cref{subsec:first_token_diversity_rank} also appears when simulating refusal training on the prompts are initially benign and not refused.
The primary prediction is that increasing refusal first-token diversity should raise refusal residuals stable rank and reduce the refusal delta after difference-in-means single vector ablation.
Attacking by ablating more principal directions is possible only if this does not degrade performance and the directions generalize across the prompts (we leave multiple-vector ablation attacks out of scope).
We keep the prompt set fixed and vary only the generated refusal-style completions by sampling 40 new-topic prompts from \textit{camel-ai/chemistry} and pairing them with refusal completions sampled from \(m\in\{1,4,8,12,16\}\) distinct first-token buckets with balanced counts within each dataset, which directly manipulates the stable rank of the refusal first-token matrix \(E_r\).
We fine-tune \texttt{OLMo-2-0425-1B-Instruct} separately on these datasets for each $m$ and select each time the earliest checkpoint that reaches or exceeds refusal rate threshold 0.6, so that later measurements are done under close base refusal rates.
For intervention, we select layer \(8\) as the earliest layer in the range of effective layers.

Additionally, for $m=8$ and $m=16$, we track refusal rates across optimization steps.
\Cref{fig:refusal_diversity_equal_step_trajectories} shows that the ablation effect is largest in a narrow optimization window.
This suggests that, although the number of optimization steps for each $m$ is different, fixing the resulting checkpoints at first-reached baseline rates of 0.6 compares largest attack affects.

\begin{figure}[h]
    \centering
    \begin{subfigure}[b]{0.31\linewidth}
        \centering
        \includegraphics[width=\linewidth]{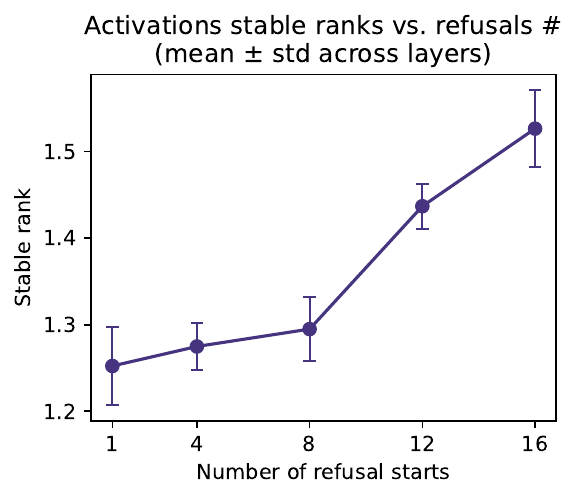}
        \label{fig:chem_stable_rank_by_diversity:aa}
    \end{subfigure}
    \begin{subfigure}[b]{0.34\linewidth}
        \centering
        \includegraphics[width=\linewidth]{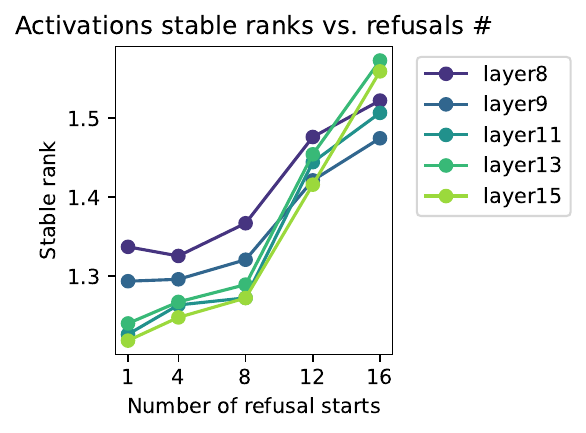}
        \label{fig:chem_stable_rank_by_diversity:a}
    \end{subfigure}
    \begin{subfigure}[b]{0.33\linewidth}
        \centering
        \includegraphics[width=\linewidth]{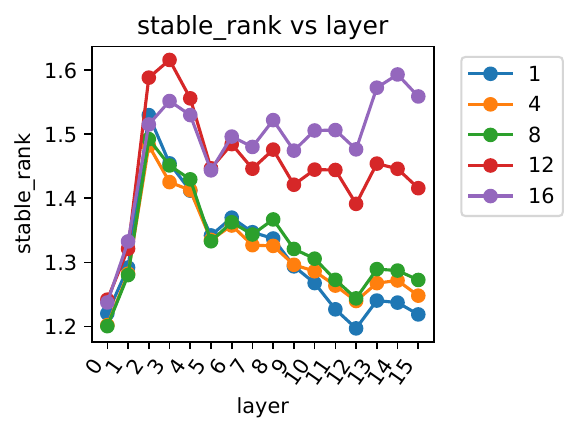}
        \label{fig:chem_stable_rank_by_diversity:b}
    \end{subfigure}
    \caption{Stable rank of refusal residuals under diverse-refusal fine-tuning. Left plot averages ranks across layers 8-15. Stable ranks across layers and their per-layer averages increase with the number of refusal first-token buckets with the clearest separation in middle-to-late layers.}
    \label{fig:chem_stable_rank_by_diversity}
\end{figure}

At the matched checkpoints, we first measure the stable rank of the resulting residuals $\Delta H$.
\Cref{fig:chem_stable_rank_by_diversity} shows that stable rank of $\Delta H$ increases with the number of refusal first-token buckets with the largest effect across mid-to-late layers.
This supports the rank mechanism: more diverse refusal starts induce less concentrated activation changes.

\begin{figure}[h]
    \centering
    \begin{subfigure}[b]{0.29\linewidth}
        \centering
        \includegraphics[width=\linewidth]{post_rebuttal_plots/4.6_batch1_onlyrefharm/____refusal_decomp___layer15.pdf}
        \label{fig:chem_refusal_delta_vs_rank:b}
    \end{subfigure}
    \begin{subfigure}[b]{0.30\linewidth}
        \centering
        \includegraphics[width=\linewidth]{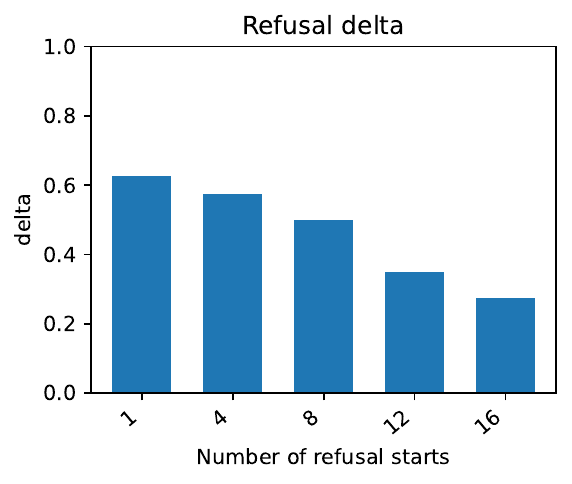}
        \label{fig:chem_refusal_delta_vs_rank:bb}
    \end{subfigure}
    \begin{subfigure}[b]{0.38\linewidth}
        \centering
        \includegraphics[width=\linewidth]{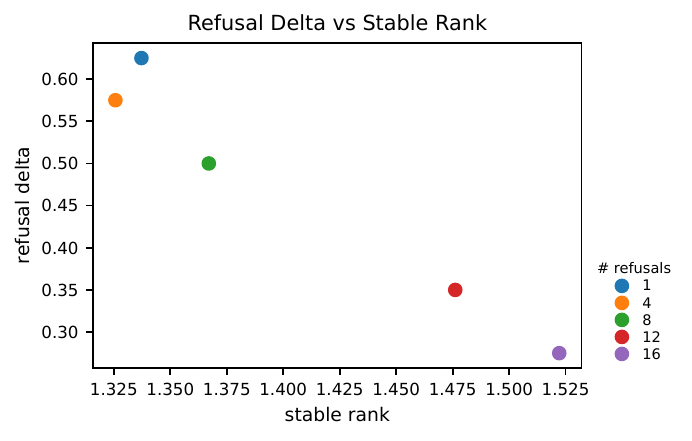}
        \label{fig:chem_refusal_delta_vs_rank:c}
    \end{subfigure}
    \caption{Refusal ablation at checkpoints matched by baseline refusal score. Left and middle plots: diverse refusals yield smaller deltas under difference-in-means ablation. Right plot: diverse refusal starts are associated with higher refusal residuals stable rank and smaller refusal deltas}
    \label{fig:chem_refusal_delta_vs_rank}
\end{figure}

We then evaluate whether this rank increase weakens single-vector ablation.
\Cref{fig:chem_refusal_delta_vs_rank} shows refusal scores before and after ablation at the matched checkpoints as well as stable rank of $\Delta H$ plotted against refusal deltas.
\Cref{fig:4_6_transfer_factors} shows that, similar to the previous sections, with diverse fine-tuning, stable rank of target quantities as well as gradient-induced update deltas increase on average, undergoing mild stable rank expansion from layer 15 to layer 7.
Finally, rightmost \Cref{fig:chem_refusal_delta_vs_rank} directly relates geometry to attackability.
More diverse refusal starts are associated with higher refusal residuals stable rank and smaller refusal deltas under ablation.
This suggests that increasing refusal first-token diversity weakens the single dominant refusal vector by spreading the fine-tuning update over more effective directions.

Overall, this controlled fine-tuning experiment suggests the rank--vulnerability pattern can hold broadly.
Even on non-harmful chemistry prompts, concentrating refusal first tokens produces lower-rank refusal residuals and stronger single-vector ablation, while diversifying refusal starts raises the refusal residuals stable ranks and reduces the ablation effect at matched baseline refusal score.

\section{Conclusion}

As AI systems enter dual-use domains, robustness failures can be catastrophic and unevenly impact at-risk communities.
Our work takes a step toward understanding AI vulnerabilities at the root cause and designing defenses from first principles.

We asked why refusal behavior can concentrate in a low-dimensional activation-space mechanism that is shared across many harmful categories and can be effectively suppressed.
In OLMo-2-0425-1B-Instruct, we find that the refusal directions reflect refusal safety-training: cross-entropy gradients from the first tokens of refusal completions induce activation updates whose mean and principal directions align with refusal direction and refusal subspace.
This connects the empirical difference-in-means refusal vector to the training updates that create it.
This mechanism suggests a simple source of brittleness exist: concentrated refusal first tokens induce low-stable-rank activation changes, making refusal behavior easier to suppress with a single-vector refusal ablation.
Across fixed-model and controlled fine-tuning experiments, increasing refusal first-token diversity raises stable rank of the activation-change matrix and is associated with smaller refusal deltas.
At the same time, refusal-completion diversity alone is unlikely a complete defense.

Overall, our results suggest that refusal robustness to ablation is partly shaped by the geometry of refusal training updates.
Token-level properties of the training completions, especially refusal-prefix concentration and first-token frequency, can shape the dimensionality of the learned refusal mechanism.
Understanding this link between data, gradients, and activation geometry may be a step toward principled mechanistic understanding of defenses against jailbreaks.
Finally, our controlled fine-tuning experiments on non-harmful datasets motivate study whether similar mechanisms hold for a broader set of features than for refusal or safety-related features.

\clearpage

\printbibliography

\appendix

\clearpage

\section{Related work}
\label{sec:related}

\subsection{Activation steering and ablation attacks}
Activation steering and ablation attacks are a family of whitebox methods which directly act on model activations and runtime to alter its output in a desired way.
Those methods are widely used to suppress refusals, bypass safety alignment, and to jailbreak open models.
An early work \parencite{turner2024steeringlanguagemodelsactivation} introduces activation engineering, an inference-time steering of model activation by adding a vector, computed from contrasting datasets.
\parencite{arditi2024refusallanguagemodelsmediated} introduced a difference-of-means jailbreak attack: the method works by estimating refusal vector as a difference between harmful and harmless prompts' activation means and later removing projections on this vector to let model respond to prompts it was trained to refuse.
\parencite{yu2025robustllmsafeguardingrefusal} additionally confirms that the operation of refusal feature ablation approximates the worst-case perturbation of offsetting model safety.
\parencite{dunefsky2025oneshotoptimizedsteeringvectors} find a steering vector by performing gradient descent from a single training example to mediate safety-relevant behavior.
\parencite{beaglehole2025universalsteeringmonitoringai} offers a universal steering by adding eigenvectors extracted from Recursive Feature Machines (RFMs) trained on contrasting datasets.

A line of work introduces sparse autoencoder (SAE) based steering. O’Brien et al. identify and amplify SAE features mediating refusal \parencite{obrien2025steeringlanguagemodelrefusal}; Bayat et al. introduce sparse activation steering by selecting SAE features from contrastive prompts \parencite{bayat2025steeringlargelanguagemodel}; SAIF steers instruction-following behavior through instruction-relevant SAE latents \parencite{he2025saifsparseautoencoderframework}.

Activation steering has been useful more broadly in non-security contexts, too: ActAdd and CAA use prompt-pair or contrastive activation differences to steer behaviors, such as topic, sentiment, factuality, or behavior \parencite{turner2024steeringlanguagemodelsactivation}, \parencite{panickssery2024steeringllama2contrastive}; PPLM steers generation by using gradients from an attribute model to act on activations \parencite{dathathri2020plugplaylanguagemodels}; Subramani et al. extract steering vectors from frozen models with gradient descent \parencite{subramani2022extractinglatentsteeringvectors}; Function Vectors and instruction-steering work extract task or instruction vectors for inference-time control \parencite{todd2024functionvectorslargelanguage}, \parencite{stolfo2025improvinginstructionfollowinglanguagemodels}.

Our work complements steering attack approaches by explaining a notable safety-critical feature in activations, the refusal vector, and its low-dimensional structure using analytical methods.

\subsection{Mechanistic interpretability, representation geometry, and safety-training dynamics}

An adjacent body of work studies the representations learned by models and the training dynamics that produce them.
\parencite{wollschläger2025geometryrefusallargelanguage} show that refusal representations span multiple orthogonal directions and can be described as a concept cone, motivating representational independence: interventions on different directions can mediate different behaviors without necessarily affecting others.
Similarly, \parencite{pan2025hiddendimensionsllmalignment} argue that refusal is multi-dimensional: a dominant direction mediates refusal, while additional components correspond to distinct interpretable safety features.
\parencite{sharkey2025openproblems} survey open problems and research areas in mechanistic interpretability, including predicting which capabilities and mechanisms arise during training or fine-tuning.
Complementing these empirical and conceptual accounts of multi-dimensional refusal, our work analytically studies the origin of such orthogonal refusal directions and links them to training-time refusal gradients.
We further provide evidence that concentrated refusal datasets can produce lower-rank gradient-induced activation updates, yielding low-stable-rank activation changes and nearly one-dimensional refusal representations; this structure, in turn, predicts larger effectiveness of the ablation attack.

A broader body of work studies concentrations as well as rank and dimensionality collapse.
\parencite{minder2025narrowfinetuningleavesclearly} demonstrate that fine-tuning on a narrow domain leaves identifiable traces of the training objective in model activations, suggesting that semantically narrow fine-tuning can induce readable activation-space biases.
\parencite{qi2024safetyalignmentjusttokens} find that the largest safety fine-tuning updates fall on the first assistant tokens, and that constraining updates to be distributed more evenly across output tokens improves robustness to fine-tuning-based jailbreaks.
\parencite{davis2026spectralgradientupdateshelp} identify conditions under which spectral optimization is preferred to Euclidean gradient descent and show that transformer activations can remain low-stable-rank during training.
\parencite{dong2021attentionnot} prove that pure self-attention can converge toward rank-one token-uniform representations without skip connections or MLPs.
\parencite{papyan2020prevalence} describe terminal-training neural collapse, in which final-layer class features collapse to class means with highly symmetric geometry, while \parencite{rangamani2023feature} extend this picture to intermediate layers, showing that deeper representations increasingly reduce within-class variance relative to between-class variance and align class-mean subspaces with dominant weight directions.
Unlike these broader investigations of rank collapse and neural collapse, our work focuses on safety post-training: we study a mechanism in which concentration of refusal first tokens makes the target matrix low stable rank, inducing low-stable-rank activations changes whose principal directions approximate the refusal subspace.


\section{Limitations and assumptions of analytical derivations}
\label{appendix_limitations}
Throughout the derivations, we make simplifying assumptions to obtain
a tractable model of the relationship between refusal targets,
gradients, and activation changes.

\paragraph{Supervised refusal-training scope.}
Our derivation models supervised fine-tuning on harmful prompts paired
with textual refusals.
We use the terms refusal training and supervised safety tuning for
this setting.
DPO, RLHF, and RLVR use different objectives and are not directly covered by the derivation which may induce slightly 
The controlled fine-tuning experiment uses effective batch size one to
isolate per-example refusal updates.
It is intended as an approximation to the sporadic
contribution of comparatively rare harmful--refusal examples within SFT dataset mixture for our target model.
Larger batches containing unrelated instruction-following examples may introduce additional signal into gradients and the resulting activation changes.

\paragraph{First-token-only loss.}
We compute all losses only at the first assistant target-token position,
after chat formatting, with the prompt positions masked.
We therefore omit gradients from the remaining completion tokens.
In the evaluated settings, the first-token loss captures substantial
refusal-direction structure, but later target positions may introduce
additional directions and training effects.

\paragraph{Deterministic per-example updates.}
We compute isolated per-example gradient steps with deterministic
backward passes and without stochastic regularization such as dropout.
These measurements do not reproduce minibatch interactions, optimizer
momentum, adaptive preconditioning, weight decay, or the accumulation
of many training steps.

\paragraph{Prompt-dependent Jacobians and conditioning.}
Exact transformer Jacobians are prompt-dependent, including through
input-dependent attention patterns.
Consequently, \Cref{thm:sr_well_conditioned} is an idealized sufficient
condition rather than a claim that complete per-example Jacobians are
shared.
Products of worst-case condition-number bounds may also become
numerically vacuous for anisotropic maps.
The reduced-map condition numbers reported in our experiments apply
only to selected data-dependent source coordinates and are not
estimates of the full prompt-specific Jacobian condition numbers.

\paragraph{In-sample shared-map fit and predictive scope.}
When \(\epsilon_{\max}\) is obtained by measuring the transfer-factor
fit error on the same \(N\) examples used to estimate
\(\widehat J_k^{L\to l}\),
Lemma~\ref{lem:stable-rank-transfer-fit-bound} gives an in-sample
fit-conditioned interval for \(\operatorname{sr}(G^{(l)})\), rather
than an independent prediction.
Specifically, it certifies that, on the fitted support, the
map-implied transfer factor \(\widehat\tau_k^{L\to l}\) agrees with the
observed transfer factor \(\tau^{L\to l}\) to relative error at most
\(\epsilon_{\max}\), and therefore that the observed target stable rank
lies within the corresponding interval.
The result becomes predictive only when the same fitted map is held
fixed and applied without refitting to additional source gradient
matrices, with its error bound established independently on held-out
source--target matrix pairs.
Moreover, agreement at the level of stable rank does not imply
matrix-level agreement; we evaluate the latter separately using
relative Frobenius error, relative spectral error, and row-wise cosine
similarity.

\section{Additional details about experimental setup.}

\paragraph{Refusal-mediating layers.}
Unless otherwise stated, the frozen-model single-layer
rank--vulnerability and ablation analyses use layer \(7\) since our measurements show that refusal vector estimation at that layer provides most effective, and therefore attacker-preferred refusal ablation.
Layerwise profiles and cross-layer transfer experiments additionally
report all explicitly labeled layers.
The controlled fine-tuning rank--vulnerability experiment in
\Cref{subsec:chem_ft_rank} uses layer \(8\) -- the earliest layer in the range of effective layers in our experiment, which retains a long last-to-middle-layer path over which to test rank propagation, similar to frozen-model experiments; selecting a later layer would make that propagation test less
stringent.

\paragraph{Estimating gradient-induced activation updates.}
To compute gradient-induced activation updates \(\mathcal G\), for
each training example \(i\) and cross-entropy loss
\(\ell_i^{\mathrm{CE}}\) on the first token of the target completion,
we apply one small parameter update
\[
\theta_{\mathrm{new}}
=
\theta-\eta\nabla_\theta\ell_i^{\mathrm{CE}},
\qquad
\eta=10^{-6}.
\]
At layer \(l\), we measure
\[
\mathfrak g_i^{(l)}
:=
-\left(
h_i^{(l)}(\theta_{\mathrm{new}})
-
h_i^{(l)}(\theta)
\right),
\]
and then restore the parameters to \(\theta\).
Thus, \(\mathfrak g_i^{(l)}\) is the negative realized activation
change from one small gradient step, consistent with our sign
convention.

For the evaluated models, we use \(l=7\) or \(l=8\), the earliest
layers at which the estimated refusal directions yield substantial
attack success; this retains a long last-to-middle-layer path over
which to test rank propagation.
All empirical \(\mathcal G\) matrices are computed from the realized
finite differences in
\Cref{eq:gradient_induced_update_vector}.

Under a local first-order expansion, their corresponding differential
is
\[
\eta J_i^{(l)}\nabla_\theta\ell_i^{\mathrm{CE}},
\]
which can instead be evaluated directly as a
Jacobian--vector product without applying a parameter
perturbation.
We leave a systematic comparison of the difference and
differential estimators to future work.

\section{Broader limitations}

\paragraph{Model, data, language, and checkpoint scope.}
This work is primarily a case study of one small OLMo model family,
relatively small prompt sets, and English-language data.
We do not identify an attackable difference-in-means refusal direction
in the base checkpoint under our estimator, but this does not rule out
other functional base-model refusal mechanisms or representations.
The released Base, SFT, DPO, RLVR1, and Instruct checkpoints differ in
their objectives, data, optimization histories, and numbers of updates.
Their comparison is therefore descriptive and does not identify which
training stage caused the observed refusal geometry or change in
ablation sensitivity.

\paragraph{Target-set-adapted white-box threat model.}
For each evaluated prompt set \(S\), the attacker estimates one difference-in-means direction from all prompts in \(S\) and uses that direction for the refusal ablation attack on \(S\).
The results therefore intentionally test an attacker-favorable same-set ablation setting and do not establish that a direction estimated from one prompt sample transfers to unseen prompts.
Held-out universality of refusal directions is a distinct question left for future work.

\paragraph{Attack-family scope.}
Our primary attack is difference-in-means single-vector ablation.
SAE-based steering, multi-feature ablation, and broader automated
representation-level attacks remain outside our evaluation.
We approach the difference-in-means single-vector attack using analytical methods and
invite application of those methods to study a broader set of attacks and defenses in mechanistic interpretability.

\paragraph{Stable rank versus causal mediation and intrinsic dimension.}
Stable rank measures global linear spectral concentration.
It does not by itself measure refusal-signal magnitude, causal
mediation strength, or local nonlinear manifold dimension.
A layer may therefore have relatively high stable rank while its
refusal direction has little behavioral effect, whereas another layer
may have lower stable rank but mediate refusal strongly.
Our rank--vulnerability comparisons concern fixed or matched
refusal-mediating layers and do not use stable rank to predict which
layer is most causally important.

\paragraph{Scope beyond refusal.}
We study refusal because it provides a shared refuse/comply mechanism
with a well-defined white-box ablation attack.
It remains unclear whether similarly universal non-refusal concepts
have the same geometry, or whether increasing feature rank generally
makes universal steering or other forms of jailbreak optimization
harder.

\section{Theorems}

\subsection{Stable rank of the refusal first-token matrix}
\label{app:Er_stable_rank}

\begin{definition}[One-hot matrix \(E_r\) for a batch of refusals]
Let \(E_r \in \RR^{N\times |V|}\) be the row-stacked one-hot matrix with \(i\)-th row \(e_{r_i}^\top\), for a batch of targets \(r=(r_1,\ldots,r_N)\).
Let \(c_k:=\#\{i:r_i=k\}\) be the count of token \(k\), let \(f_k:=c_k/N\), and let \(m:=|\{k:c_k>0\}|\).
\end{definition}

\begin{lemma}[Spectrum of a stacked one-hot matrix]
\label[lemma]{lem:Er_spectrum_appendix}
The feature-space Gram matrix of \(E_r\) is diagonal:
\begin{align}
E_r^\top E_r
=
\operatorname{diag}(c_1,\ldots,c_{|V|}).
\end{align}
Consequently, the nonzero singular values of \(E_r\) are \(\sqrt{c_k}\) for tokens with \(c_k>0\), and
\begin{align}
\rank(E_r)=m.
\end{align}
The stable rank is
\begin{align}
\sr(E_r)
=
\frac{\|E_r\|_F^2}{\|E_r\|_2^2}
=
\frac{\sum_{k=1}^{|V|}c_k}{\max_k c_k}
=
\frac{N}{\max_k c_k}
=
\frac{1}{\max_k f_k}.
\end{align}
\end{lemma}

\begin{proof}
The proof uses an approach similar to \parencite{davis2026spectralgradientupdateshelp}.
For any \(k,k'\in V\),
\begin{align}
(E_r^\top E_r)_{k,k'}
=
\sum_{i=1}^N (E_r)_{i,k}(E_r)_{i,k'}
=
\sum_{i=1}^N \mathbf{1}[r_i=k]\mathbf{1}[r_i=k'].
\end{align}
If \(k\neq k'\), no example can satisfy both \(r_i=k\) and \(r_i=k'\), so the sum is \(0\).
If \(k=k'\), the sum is \(c_k\).
Thus \(E_r^\top E_r=\operatorname{diag}(c_1,\ldots,c_{|V|})\).
The eigenvalues of \(E_r^\top E_r\) are \(c_k\), so the singular values of \(E_r\) are \(\sqrt{c_k}\).
The rank is the number of positive counts.
Finally, \(\|E_r\|_F^2=\sum_k c_k=N\), while \(\|E_r\|_2^2=\max_k c_k\), giving the stable-rank formula.
\end{proof}

\begin{corollary}[Minimizer and maximizer of \(\sr(E_r)\)]
\label{cor:Er_stable_rank_extrema_appendix}
Let \(c_{\max}:=\max_k c_k\).
Then
\begin{align}
1
\le
\sr(E_r)
=
\frac{N}{c_{\max}}
\le
m.
\end{align}
Moreover:
\begin{enumerate}
\item \(\sr(E_r)=1\) iff \(c_{\max}=N\), i.e., all targets share the same first token.
\item For fixed support size \(m\), \(\sr(E_r)\) is maximized when counts are as balanced as possible across the \(m\) observed tokens, i.e. \(c_k\in\{\lfloor N/m\rfloor,\lceil N/m\rceil\}\) on the support.
In that case,
\begin{align}
\max \sr(E_r)
=
\frac{N}{\lceil N/m\rceil}.
\end{align}
If \(m\mid N\), this maximum equals \(m\).
\item If the support is unconstrained and \(N\le |V|\), the global maximum is \(\sr(E_r)=N\), achieved when all first tokens are distinct.
If \(N>|V|\), the maximum is achieved by balancing counts over all \(|V|\) vocabulary tokens.
\end{enumerate}
\end{corollary}

\begin{proof}
By \Cref{lem:Er_spectrum_appendix}, \(\sr(E_r)=N/c_{\max}\).
The lower bound follows from \(c_{\max}\le N\), with equality iff \(c_{\max}=N\).
For fixed support size \(m\), the pigeonhole principle gives \(c_{\max}\ge \lceil N/m\rceil\), hence
\begin{align}
\sr(E_r)
=
\frac{N}{c_{\max}}
\le
\frac{N}{\lceil N/m\rceil}
\le
m.
\end{align}
This is tight when counts are balanced across the \(m\) observed targets.
The unconstrained-support statement follows by taking \(m=\min\{N,|V|\}\).
\end{proof}

\subsection{Stable rank propagation theorems}
\label{app:sr_conditioned_factor}

The following proof uses standard matrix norm inequalities: the Frobenius/Schatten-\(2\) ideal property, operator norm submultiplicativity, and the definition of the spectral condition number; see, e.g., \parencite[Chapter~5]{horn2012matrixanalysis}.

\begin{theorem}[Stable rank under multiplication by a well-conditioned factor]
\label{thm:sr_well_conditioned_appendix}
Let \(A\in\RR^{N\times n}\), let \(B\in\RR^{n\times n}\) be invertible, and define condition number:
\begin{align}
\kappa(B):=\|B\|_2\|B^{-1}\|_2.
\end{align}
Then
\begin{align}
\frac{1}{\kappa(B)^2}\sr(A)
\le
\sr(AB)
\le
\kappa(B)^2\sr(A).
\end{align}
\end{theorem}

\begin{proof}
Recall that
\begin{align}
\sr(M):=\frac{\|M\|_F^2}{\|M\|_2^2}.
\end{align}
We use
\begin{align}
\|AB\|_F &\le \|A\|_F\|B\|_2, \\
\|AB\|_2 &\le \|A\|_2\|B\|_2.
\end{align}
Since \(A=(AB)B^{-1}\),
\begin{align}
\|A\|_2
=
\|(AB)B^{-1}\|_2
\le
\|AB\|_2\|B^{-1}\|_2,
\end{align}
so
\begin{align}
\|AB\|_2
\ge
\frac{\|A\|_2}{\|B^{-1}\|_2}.
\end{align}
Therefore,
\begin{align}
\sr(AB)
=
\frac{\|AB\|_F^2}{\|AB\|_2^2}
\le
\frac{\left(\|A\|_F\|B\|_2\right)^2}{\left(\|A\|_2/\|B^{-1}\|_2\right)^2}
=
\kappa(B)^2\sr(A).
\end{align}
For the lower bound, apply the upper bound to \(A=(AB)B^{-1}\):
\begin{align}
\sr(A)
\le
\kappa(B^{-1})^2\sr(AB)
=
\kappa(B)^2\sr(AB).
\end{align}
Rearranging gives
\begin{align}
\sr(AB)
\ge
\frac{1}{\kappa(B)^2}\sr(A).
\end{align}
\end{proof}

\begin{lemma}[Stable rank bounds from transfer factor fit error]
\label{lem:stable-rank-transfer-fit-bound-proof}

Let \(G^{(L)},G^{(l)}\in\mathbb{R}^{N\times d}\) be nonzero gradient
matrices, and let \(\widehat{J}_k^{L\to l}\) be a top-$k$ source-subspace
shared map with fitted output
\[
\widehat{G}_k^{(l)}
=
G^{(L)}\widehat{J}_k^{L\to l}.
\]
Suppose that the relative transfer-factor fit error is bounded by some $\epsilon_{max}$:
\[
\epsilon_{\tau,k}^{L\to l} \leq \epsilon_{max} <1.
\]
Then
\[
\frac{
    \widehat{\tau}_k^{L\to l}
}{
    1+\epsilon_{max}
}
\operatorname{sr}\!\left(G^{(L)}\right)
\leq
\operatorname{sr}\!\left(G^{(l)}\right)
\leq
\frac{
    \widehat{\tau}_k^{L\to l}
}{
    1-\epsilon_{max}
}
\operatorname{sr}\!\left(G^{(L)}\right).
\]
\end{lemma}

\begin{proof}
By the definition of the relative transfer-factor fit error and assuming it is bounded by $\epsilon_{max}$,
\[
\left|
\widehat{\tau}_k^{L\to l}
-
\tau^{L\to l}
\right|
\leq
\epsilon_{max}\tau^{L\to l}.
\]
Hence,
\[
\left(1-\epsilon_{max}\right)\tau^{L\to l}
\leq
\widehat{\tau}_k^{L\to l}
\leq
\left(1+\epsilon_{max}\right)\tau^{L\to l}.
\]
Since \(\epsilon_{max}<1\), rearranging gives
\[
\frac{
    \widehat{\tau}_k^{L\to l}
}{
    1+\epsilon_{max}
}
\leq
\tau^{L\to l}
\leq
\frac{
    \widehat{\tau}_k^{L\to l}
}{
    1-\epsilon_{max}
}.
\]
Finally, using
\[
\operatorname{sr}\left(G^{(l)}\right)
=
\tau^{L\to l}
\operatorname{sr}\left(G^{(L)}\right)
\]
proves the result.
\end{proof}

\section{Matrix metrics}
\label{app:matrix_metrics}

This appendix defines the matrix metrics used in the empirical sections.
All matrices below are row-stacked matrices in \(\RR^{n\times d}\), such as benign-centered refusal residuals \(\Delta H\) or gradient-induced activation update matrices \(\mathcal G^{(l)}\).

\paragraph{Thin SVD and right-singular subspaces.}
For \(M\in\RR^{n\times d}\), let
\[
M=U_M\Sigma_MV_M^\top
\]
denote its thin SVD, where \(V_M\in\RR^{d\times r_M}\) contains the right singular vectors associated with nonzero singular values and \(r_M:=\rank(M)\).
We write
\[
\mathcal V(M):=\operatorname{span}(V_M)
\]

\paragraph{Uncentered single-matrix metrics.}
For metrics of a single matrix, we use the uncentered matrix \(M\).
The empirical second-moment matrix is
\[
C_M:=\frac{1}{n}M^\top M,
\qquad
\lambda_{M,i}:=\frac{\sigma_{M,i}^2}{n}.
\]
The stable rank is
\[
\sr(M)
:=
\frac{\|M\|_F^2}{\|M\|_2^2}
=
\frac{\sum_{i=1}^{r_M}\sigma_{M,i}^2}{\sigma_{M,1}^2}.
\]

\paragraph{Centered pairwise comparisons between right singular subspaces.}
When comparing right singular subspaces between matrices, we use centered matrices, which means that for each matrix $M$, we remove its row mean $\mu_M$ to obtain its centered matrix $M_c$ before computing its right singular vectors:
\[
\mu_M = \frac{1}{n}M^\top\mathbf 1_n,
\qquad
M_c = M-\mathbf 1_n\mu_M^\top.
\]
Unless otherwise stated, \(V_{M,k}\) and \(\lambda_{M,i}\) in pairwise metrics are computed from the centered matrix \(M_c\).
This separates the mean direction, measured by cosine similarity, from the remaining PCA-style subspace structure.

\paragraph{Top-\(k\) policy.}
For two matrices \(A,B\), let \(r_A:=\rank(A_c)\) and \(r_B:=\rank(B_c)\).
Given a requested \(k_{\mathrm{top}}\), we use
\[
k:=\min\{k_{\mathrm{top}},r_A,r_B\}.
\]
In the main experiments we set \(k_{\mathrm{top}}=100\).
Since \(n < 100\) for our experiments, this includes all nonzero centered principal directions after the rank cap.

\paragraph{Cosine similarity of row means.}
For \(A,B\in\RR^{n\times d}\), define
\[
\operatorname{cosine\_mean}(A,B)
:=
\frac{\mu_A^\top\mu_B}{\|\mu_A\|_2\|\mu_B\|_2}
\in[-1,1].
\]
This metric compares only the mean directions and is complementary to centered subspace-overlap metrics.

\paragraph{Explained-variance overlap.}
Let \(A_c=U_A\Sigma_AV_A^\top\) and \(B_c=U_B\Sigma_BV_B^\top\) be centered thin SVDs.
Let \(V_{A,k}\) and \(V_{B,k}\) denote the first \(k\) right singular vectors, and define
\[
C^{(k)}_{A,B}:=V_{A,k}^\top V_{B,k}.
\]
The explained-variance overlap of \(A\)'s top-\(k\) directions captured by \(B\)'s top-\(k\) subspace is
\[
\operatorname{EV}_k(A\mid B)
:=
\sum_{i=1}^k
\frac{\lambda_{A,i}}{\sum_{\ell=1}^k\lambda_{A,\ell}}
\sum_{j=1}^k
\left(C^{(k)}_{A,B}\right)_{ij}^2
\in[0,1].
\]
Equivalently, each \(A\)-principal direction is weighted by its share of \(A\)'s top-\(k\) variance and scored by the squared norm of its projection onto \(\operatorname{span}(V_{B,k})\).
The metric is directional because it weights only by \(A\)'s spectrum.

\paragraph{Average principal-angle overlap.}
Let
\[
A_c = U_A\Sigma_A V_A^\top,
\qquad
B_c = U_B\Sigma_B V_B^\top
\]
be centered thin SVDs. Let \(V_{A,k}\) and \(V_{B,k}\) contain the
first \(k\) right singular vectors, and define
\[
C_{A,B}^{(k)}
:=
V_{A,k}^\top V_{B,k}.
\]
Let \(s_i(C_{A,B}^{(k)})\) denote the singular values of
\(C_{A,B}^{(k)}\). We define the average principal-angle overlap as
\begin{align}
\operatorname{AvgOverlap}_k(A,B)
&:=
\frac{1}{k}
\sum_{i=1}^{k}
s_i\!\left(C_{A,B}^{(k)}\right)^2
\\
&=
\frac{1}{k}
\left\|
V_{A,k}^\top V_{B,k}
\right\|_F^2
\in [0,1].
\end{align}
The singular values of \(C_{A,B}^{(k)}\) are the cosines of the
principal angles between the two subspaces. Thus,
\(\operatorname{AvgOverlap}_k\) is symmetric and basis-invariant,
equals \(1\) when the top-\(k\) subspaces coincide, and equals \(0\)
when they are orthogonal.

\paragraph{Shared-map metrics.}
All shared-map metrics below are computed on uncentered gradient
matrices. Let \(G^{(L)},G^{(l)}\in\mathbb{R}^{N\times d}\) be the
source and target gradient matrices, and let
\[
\widehat G_k^{(l)}
=
G^{(L)}\widehat J_k^{L\to l}
\]
be the output of the fitted top-\(k\) source-subspace shared map.

\paragraph{Source-subspace energy capture.}
Let \(V_k^{(L)}\in\mathbb{R}^{d\times k}\) contain the top-\(k\)
right singular vectors of \(G^{(L)}\). We define the fraction of
source-gradient energy contained in the selected subspace as
\begin{align}
\operatorname{EC}_k^{(L)}
:=
\frac{
    \left\|G^{(L)}V_k^{(L)}\right\|_F^2
}{
    \left\|G^{(L)}\right\|_F^2
}.
\end{align}
Equivalently, if
\(\sigma_1(G^{(L)})\geq\sigma_2(G^{(L)})\geq\cdots\) are the
singular values of \(G^{(L)}\), then
\begin{align}
\operatorname{EC}_k^{(L)}
=
\frac{
    \sum_{j=1}^{k}\sigma_j\!\left(G^{(L)}\right)^2
}{
    \sum_j\sigma_j\!\left(G^{(L)}\right)^2
}.
\end{align}
This metric measures how much of the total squared source-gradient
energy is retained before fitting the shared map.

\paragraph{Relative Frobenius reconstruction error.}
We define the relative Frobenius error between the fitted and observed
target matrices as
\begin{align}
\operatorname{RelF}_k^{L\to l}
:=
\frac{
    \left\|
    G^{(l)}-\widehat G_k^{(l)}
    \right\|_F
}{
    \left\|G^{(l)}\right\|_F
}.
\end{align}
This metric measures the total reconstruction error relative to the
Frobenius norm of the observed target matrix. A value of zero denotes
exact reconstruction; the metric is not upper bounded by one.

\paragraph{Relative spectral reconstruction error.}
We define the relative spectral error as
\begin{align}
\operatorname{Rel2}_k^{L\to l}
:=
\frac{
    \left\|
    G^{(l)}-\widehat G_k^{(l)}
    \right\|_2
}{
    \left\|G^{(l)}\right\|_2
}.
\end{align}
This metric measures the largest residual singular direction relative
to the dominant singular direction of the observed target matrix. It
complements the relative Frobenius error because stable rank depends
on both the Frobenius and spectral norms.

\clearpage

\subsection{Refusal direction across released post-training checkpoints - additional plots}

\begin{figure}[h]
    \centering
    \begin{subfigure}[t]{0.45\linewidth}
        \centering
        \includegraphics[width=\linewidth]{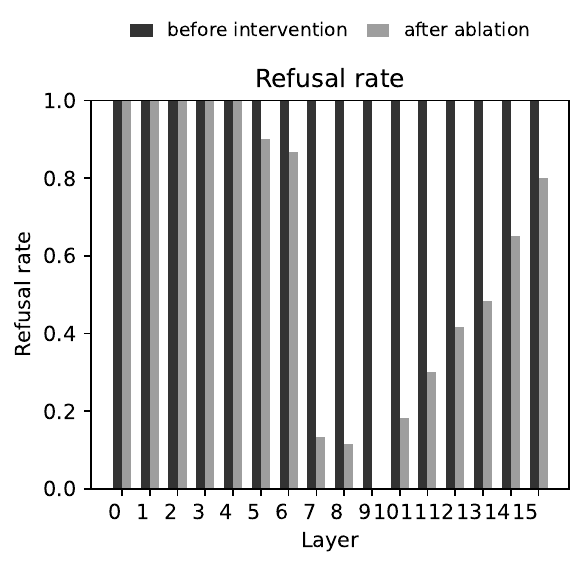}
        \label{fig:_profiler_to_upload2_additional_layer:b}
    \end{subfigure}
    \begin{subfigure}[t]{0.48\linewidth}
        \centering
        \includegraphics[width=\linewidth]{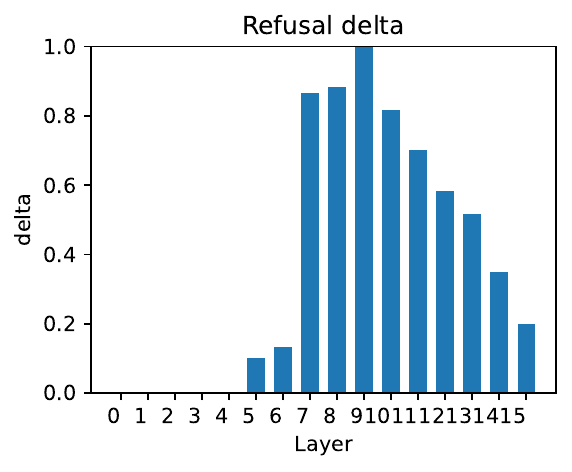}
        \label{fig:_profiler_to_upload2_additional_layer:bb}
    \end{subfigure}
    \caption{Refusal ablation attack effectiveness per each layer used to estimate the refusal difference-in-means direction (for a fixed model OLMo-2-0425-1B-Instruct against a subset of AdvBench prompts).
    The plot shows that fixing layers to 7-10 provides the largest ablation attack success (refusal delta) with decreasing success across later layers, while earlier layers provide almost no benefit, resulting in relatively ineffective attacks.}
    \label{fig:_profiler_to_upload2_additional_layer}
\end{figure}

\begin{figure}[h]
    \begin{subfigure}[b]{0.33\linewidth}
        \centering
        \includegraphics[width=\linewidth]{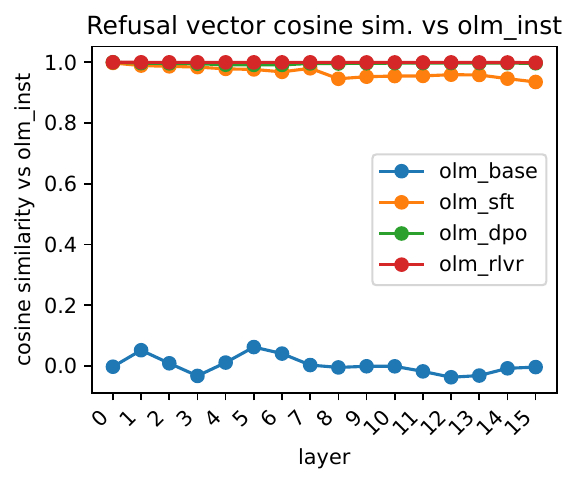}
        \label{fig:_profiler_to_upload2____refusal_decomp___layer15__pairwise:c}
    \end{subfigure}
    \begin{subfigure}[b]{0.33\linewidth}
        \centering
        \includegraphics[width=\linewidth]{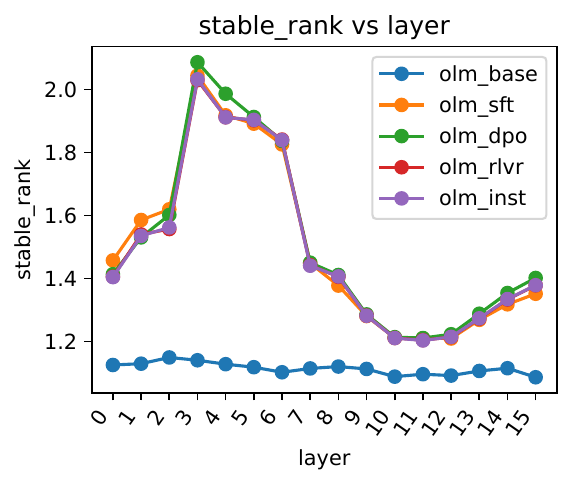}
        \label{fig:_profiler_to_upload2____refusal_decomp___layer15__pairwise:d}
    \end{subfigure}
    \begin{subfigure}[b]{0.33\linewidth}
        \centering
        \includegraphics[width=\linewidth]{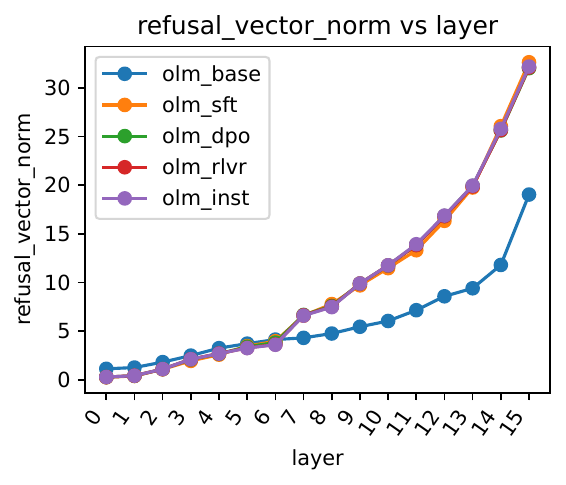}
        \label{fig:_profiler_to_upload2____refusal_decomp___layer15__pairwise:a}
    \end{subfigure}
    \caption{Measuring refusal vectors across stages shows that once a functional refusal vector appears at SFT stage, all resulting refusal directions across all stages (except the base model) are all similar for a subset of selected AdvBench prompts. Refusal vector norms grow after SFT stages and stay similar throughout post-training stages.}
    \label{fig:_profiler_to_upload2____refusal_decomp___layer15__pairwise}
\end{figure}

\clearpage

\section{Refusal first-token diversity increases gradient stable ranks - additional plots}

\begin{figure}[h]
    \centering
    \begin{subfigure}[t]{0.31\linewidth}
        \centering
        \includegraphics[width=\linewidth]{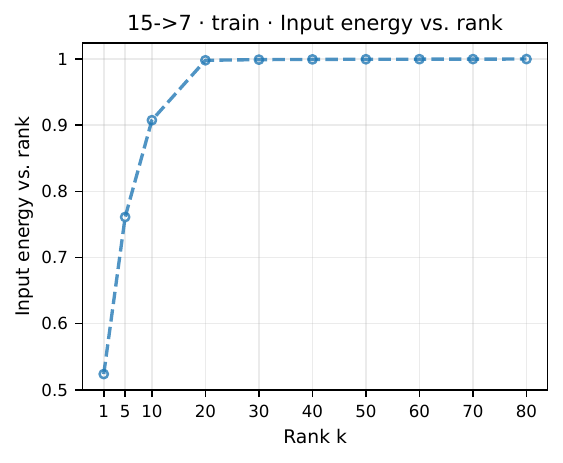}
        \caption{Source-energy capture.}
        \label{fig:4_4_shared_map:energy}
    \end{subfigure}
    \begin{subfigure}[t]{0.31\linewidth}
        \centering
        \includegraphics[width=\linewidth]{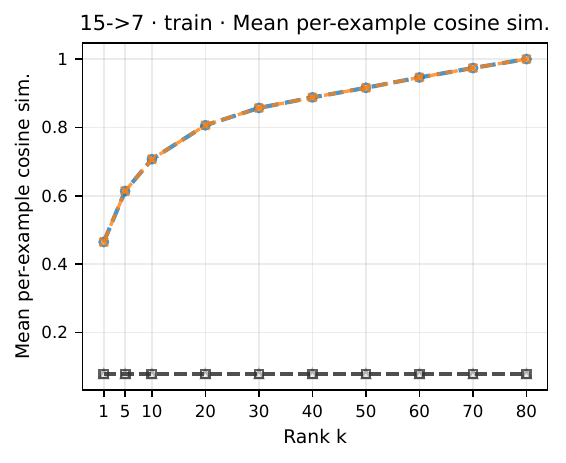}
        \caption{Row-wise cosine similarity.}
        \label{fig:4_4_shared_map:cosine}
    \end{subfigure}
    \begin{subfigure}[t]{0.31\linewidth}
        \centering
        \includegraphics[width=\linewidth]{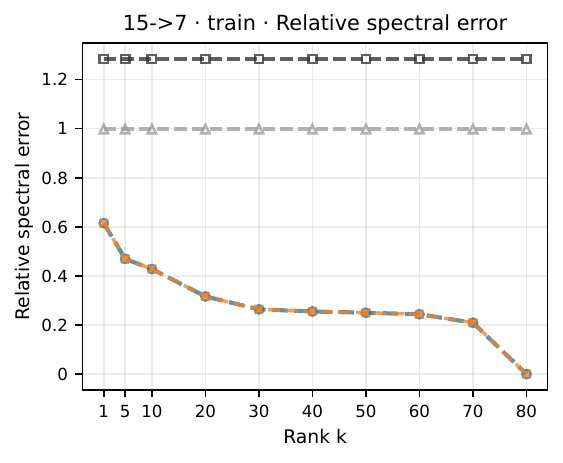}
        \caption{Relative spectral error.}
        \label{fig:4_4_shared_map:rel2}
    \end{subfigure}

    \begin{subfigure}[t]{0.40\linewidth}
        \centering
        \includegraphics[width=\linewidth]{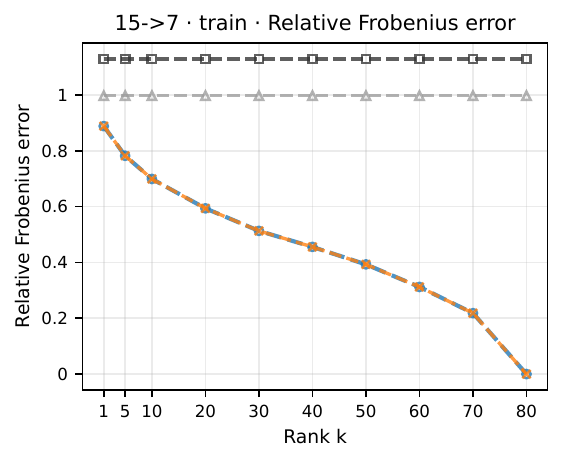}
        \caption{Relative Frobenius error.}
        \label{fig:4_4_shared_map:relf}
    \end{subfigure}
    \begin{subfigure}[t]{0.40\linewidth}
        \centering
        \includegraphics[width=\linewidth]{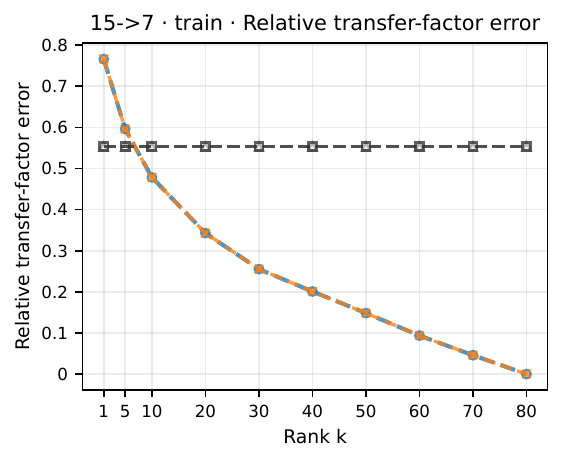}
        \caption{Transfer-factor fit error.}
        \label{fig:4_4_shared_map:tau}
    \end{subfigure}

    \caption{
    In-sample fit quality of the top-\(k\) source-subspace shared map from \(G^{(15)}\) to \(G^{(7)}\) for the \(m=16\) refusal-start condition.
    Increasing \(k\) captures more source-gradient energy, improves per-example directional agreement, and reduces matrix-reconstruction and stable-rank-transfer errors.
    At \(k=20\), the map captures \(99.8\%\) of source energy and attains relative Frobenius, spectral, and transfer-factor errors \(0.594\), \(0.316\), and \(0.343\).
    The \(k=80\) endpoint is the in-sample interpolation limit and is shown only as a reference, not as evidence of a low-rank mechanism.
    }
    \label{fig:4_4_shared_map}
\end{figure}

\begin{figure}[h]
    \centering
    \begin{subfigure}[t]{0.31\linewidth}
        \centering
        \includegraphics[width=\linewidth]{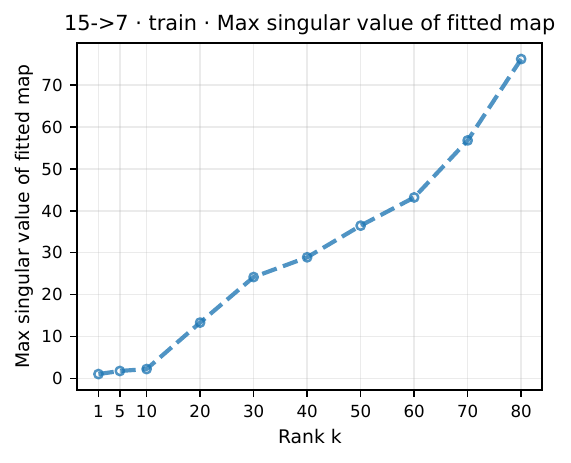}
        \caption{Largest singular value.}
        \label{fig:4_4_shared_map_spectrum:smax}
    \end{subfigure}
    \begin{subfigure}[t]{0.31\linewidth}
        \centering
        \includegraphics[width=\linewidth]{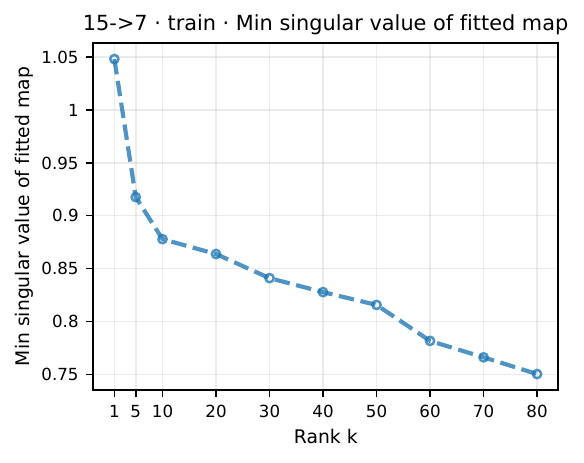}
        \caption{Smallest nonzero singular value.}
        \label{fig:4_4_shared_map_spectrum:smin}
    \end{subfigure}
    \begin{subfigure}[t]{0.31\linewidth}
        \centering
        \includegraphics[width=\linewidth]{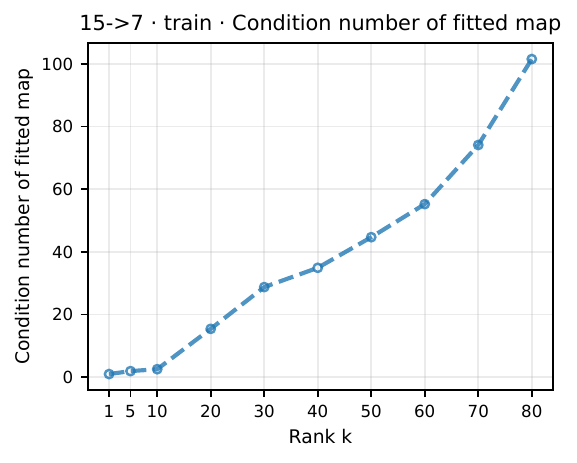}
        \caption{Reduced-map condition number.}
        \label{fig:4_4_shared_map_spectrum:kappa}
    \end{subfigure}

    \caption{
    Singular-value diagnostics for the fitted map on its selected source coordinates.
    The smallest nonzero singular value changes little between \(k=10\) and \(k=20\), from \(0.878\) to \(0.864\), while the largest singular value grows from \(2.24\) to \(13.33\).
    Consequently, the reduced-map condition number rises from \(2.55\) to \(15.43\).
    This explains why the higher-rank map gives a better empirical fit while its worst-case condition-number bound becomes less informative.
    }
    \label{fig:4_4_shared_map_spectrum}
\end{figure}

\clearpage

\section{Refusal first-token diversity increases residual rank and weakens single-vector ablation - additional plots}

\begin{figure}[h]
    \centering
    \begin{subfigure}[b]{0.4\linewidth}
        \centering
        \includegraphics[width=\linewidth]{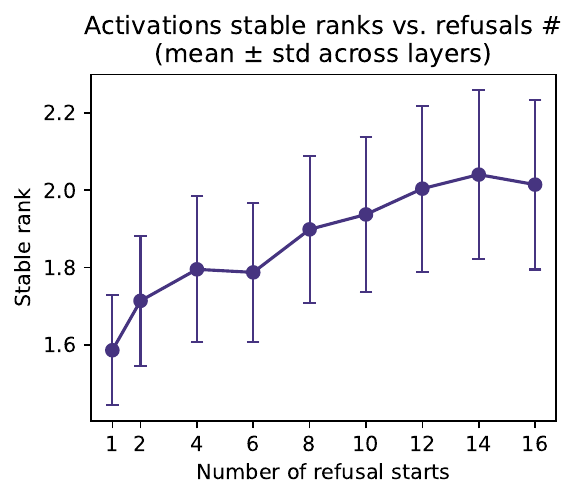}
        \label{fig:____stable_rank_vs_x__layerwise_overlay_additional:a}
    \end{subfigure}
    \caption{Averaged ranks over layers 7-15 stable rank of benign-centered refusal residuals \(\Delta H\). Increasing refusal first-token diversity raises \(\operatorname{sr}(\Delta H)\) in middle-to-late layers.}
    \label{fig:____stable_rank_vs_x__layerwise_overlay_additional}
\end{figure}

\begin{figure}[h]
    \centering
    \begin{subfigure}[t]{0.37\linewidth}
        \centering
        \includegraphics[width=\linewidth]{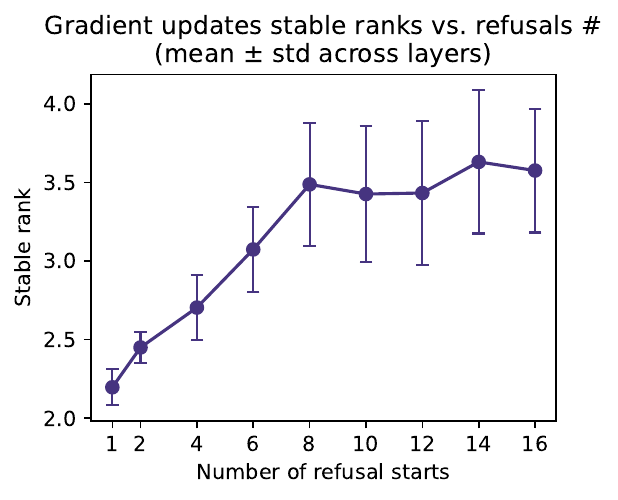}
        \label{fig:____stable_rank_grad_vs_x__layerwise_overlay:aa}
    \end{subfigure}
    \begin{subfigure}[t]{0.42\linewidth}
        \centering
        \includegraphics[width=\linewidth]{post_rebuttal_plots/4.5_batch1_only_refharm/____stable_rank_grad_ds_vs_x__layerwise_overlay.pdf}
        \label{fig:____stable_rank_grad_vs_x__layerwise_overlay:a}
    \end{subfigure}
    \begin{subfigure}[t]{0.45\linewidth}
        \centering
        \includegraphics[width=\linewidth]{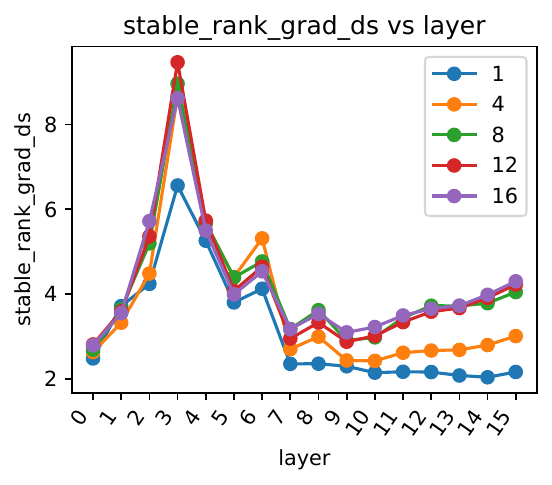}
        \label{fig:____stable_rank_grad_vs_x__layerwise_overlay:b}
    \end{subfigure}
    \caption{Stable rank of gradient-induced update deltas \(\mathcal G\): more refusal first-token buckets produce higher effective rank across most layers, although the increase saturates at the end. Left upper plot averages stable ranks across layers 7-15.}
    \label{fig:____stable_rank_grad_vs_x__layerwise_overlay}
\end{figure}

\newpage

\subsection{Frozen-model raw-gradient shared-map diagnostics}
\label{sec:4_5_shared_map_gradients}
\begin{figure}[h]
    \centering
    \begin{subfigure}[t]{0.31\linewidth}
        \centering
        \includegraphics[width=\linewidth]{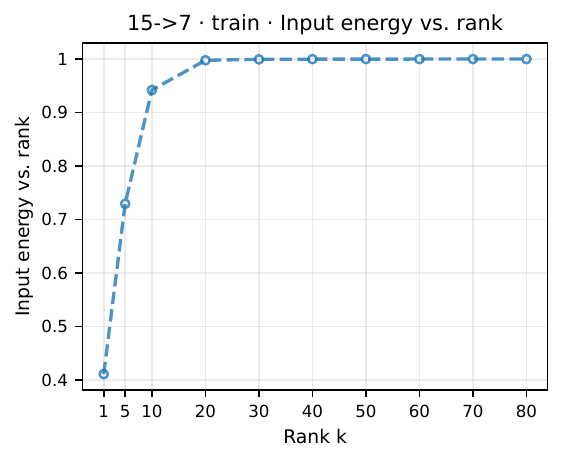}
        \caption{Source-energy capture.}
        \label{fig:4_5_shared_map_gradients:energy}
    \end{subfigure}
    \begin{subfigure}[t]{0.31\linewidth}
        \centering
        \includegraphics[width=\linewidth]{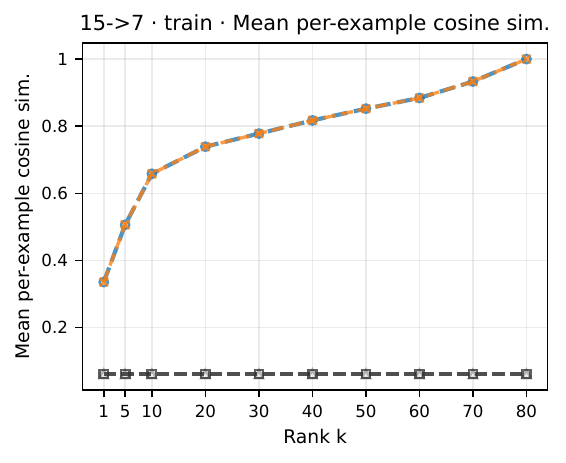}
        \caption{Row-wise cosine similarity.}
        \label{fig:4_5_shared_map_gradients:cosine}
    \end{subfigure}
    \begin{subfigure}[t]{0.31\linewidth}
        \centering
        \includegraphics[width=\linewidth]{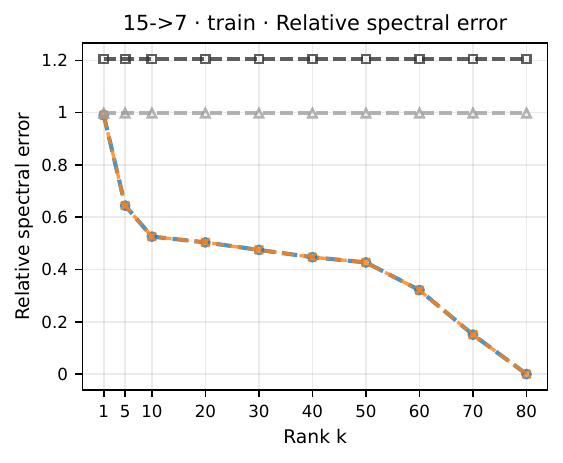}
        \caption{Relative spectral error.}
        \label{fig:4_5_shared_map_gradients:rel2}
    \end{subfigure}

    \begin{subfigure}[t]{0.33\linewidth}
        \centering
        \includegraphics[width=\linewidth]{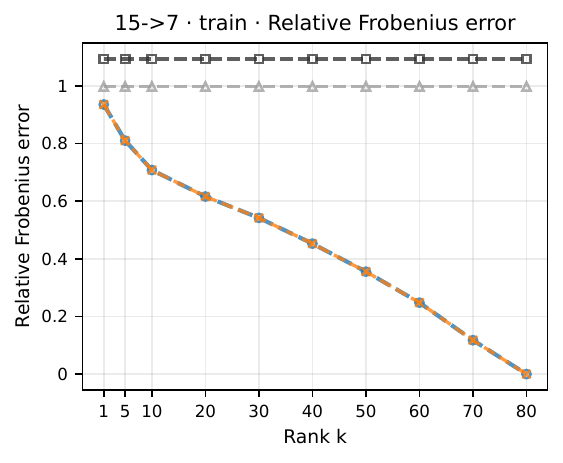}
        \caption{Relative Frobenius error.}
        \label{fig:4_5_shared_map_gradients:relf}
    \end{subfigure}
    \begin{subfigure}[t]{0.33\linewidth}
        \centering
        \includegraphics[width=\linewidth]{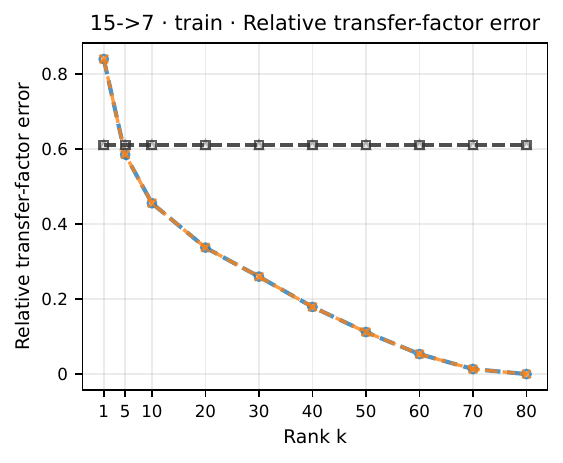}
        \caption{Transfer-factor fit error.}
        \label{fig:4_5_shared_map_gradients:tau}
    \end{subfigure}

    \caption{
    In-sample fit quality of the top-\(k\) source-subspace shared map from \(G^{(15)}\) to \(G^{(7)}\) for the \(m=16\) prompt-subset condition.
    Increasing \(k\) captures more source-gradient energy, improves per-example directional agreement, and reduces matrix-reconstruction and stable-rank-transfer errors.
    At \(k=20\), the selected source subspace captures \(99.8\%\) of source-gradient energy.
    The relative Frobenius and spectral errors are \(0.616\) and \(0.504\), the mean row-wise cosine similarity is \(0.739\), and the relative transfer-factor fit error is \(0.337\).
    Thus, the rank-20 source subspace contains nearly all source-gradient energy, while the remaining reconstruction errors show that the shared map captures a meaningful but incomplete component of the observed cross-layer relation.
    The \(k=80\) endpoint is the in-sample interpolation limit and is shown only as a reference, not as evidence of a low-rank mechanism.
    }
    \label{fig:4_5_shared_map_gradients}
\end{figure}
\begin{figure}[h]
    \centering
    \begin{subfigure}[t]{0.31\linewidth}
        \centering
        \includegraphics[width=\linewidth]{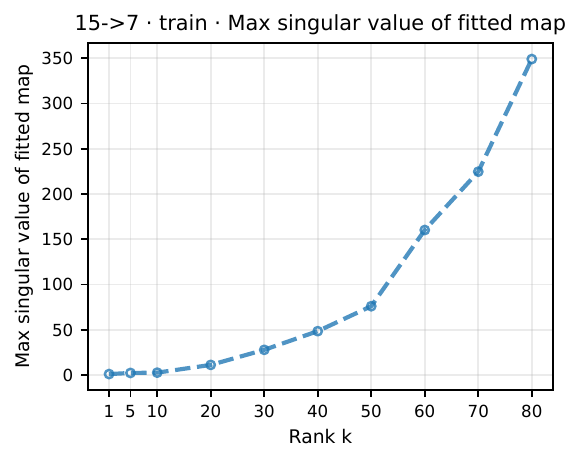}
        \caption{Largest singular value.}
        \label{fig:4_5_shared_map_gradients_spectrum:smax}
    \end{subfigure}
    \begin{subfigure}[t]{0.31\linewidth}
        \centering
        \includegraphics[width=\linewidth]{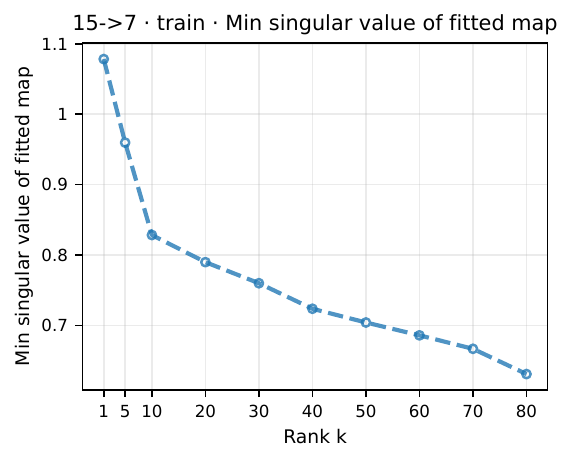}
        \caption{Smallest nonzero singular val.}
        \label{fig:4_5_shared_map_gradients_spectrum:smin}
    \end{subfigure}
    \begin{subfigure}[t]{0.31\linewidth}
        \centering
        \includegraphics[width=\linewidth]{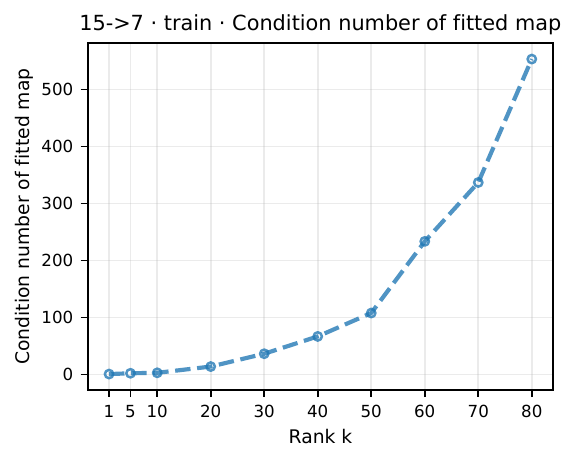}
        \caption{Reduced-map condition num.}
        \label{fig:4_5_shared_map_gradients_spectrum:kappa}
    \end{subfigure}

    \caption{
    Singular-value diagnostics for the reduced coordinate map fitted to the raw gradients.
    Between \(k=10\) and \(k=20\), the smallest nonzero singular value decreases slightly from \(0.83\) to \(0.79\), the largest singular value grows from \(2.65\) to \(11.19\), the reduced-map condition number rises from \(3.19\) to \(14.17\).
    The higher-rank map gives a better in-sample fit but becomes substantially more anisotropic, making its worst-case condition-number bound less informative.
    }
    \label{fig:4_5_shared_map_gradients_spectrum}
\end{figure}

\clearpage

\subsection{Frozen-model gradient-update shared-map diagnostics}
\label{sec:4_5_shared_map_updates}

\begin{figure}[h]
    \centering
    \begin{subfigure}[t]{0.31\linewidth}
        \centering
        \includegraphics[width=\linewidth]{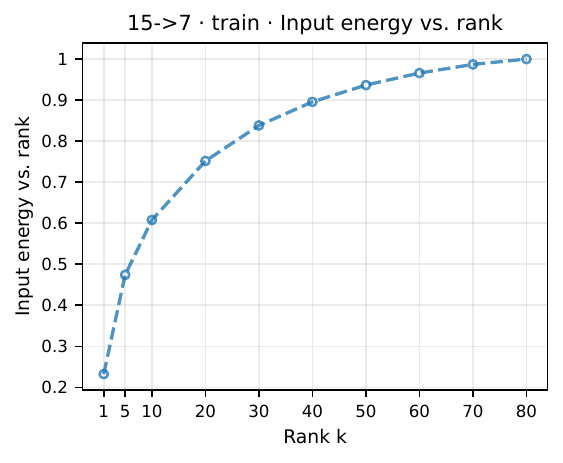}
        \caption{Source-energy capture.}
        \label{fig:4_5_shared_map_updates:energy}
    \end{subfigure}
    \begin{subfigure}[t]{0.31\linewidth}
        \centering
        \includegraphics[width=\linewidth]{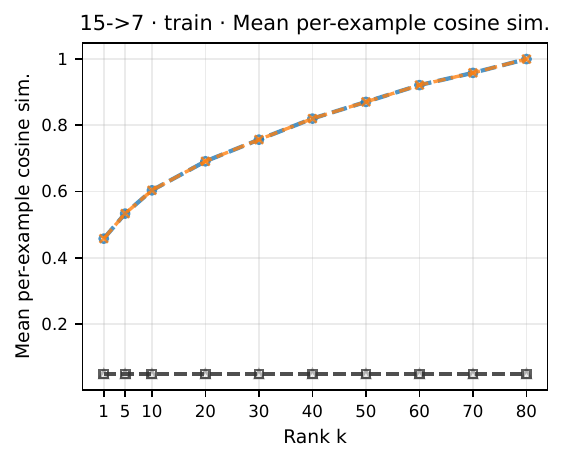}
        \caption{Row-wise cosine similarity.}
        \label{fig:4_5_shared_map_updates:cosine}
    \end{subfigure}
    \begin{subfigure}[t]{0.31\linewidth}
        \centering
        \includegraphics[width=\linewidth]{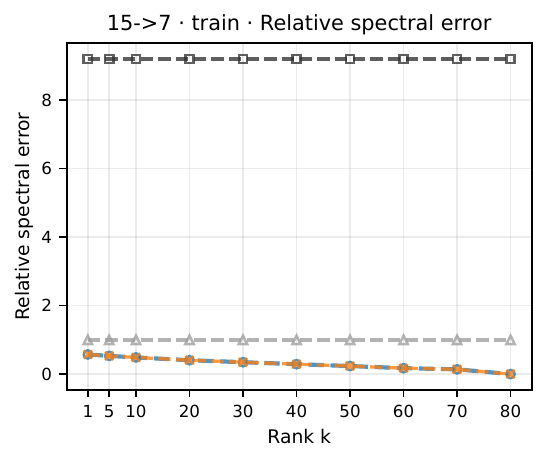}
        \caption{Relative spectral error.}
        \label{fig:4_5_shared_map_updates:rel2}
    \end{subfigure}

    \begin{subfigure}[t]{0.33\linewidth}
        \centering
        \includegraphics[width=\linewidth]{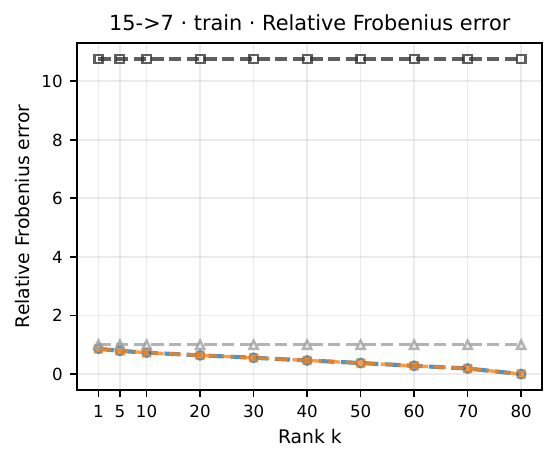}
        \caption{Relative Frobenius error.}
        \label{fig:4_5_shared_map_updates:relf}
    \end{subfigure}
    \begin{subfigure}[t]{0.33\linewidth}
        \centering
        \includegraphics[width=\linewidth]{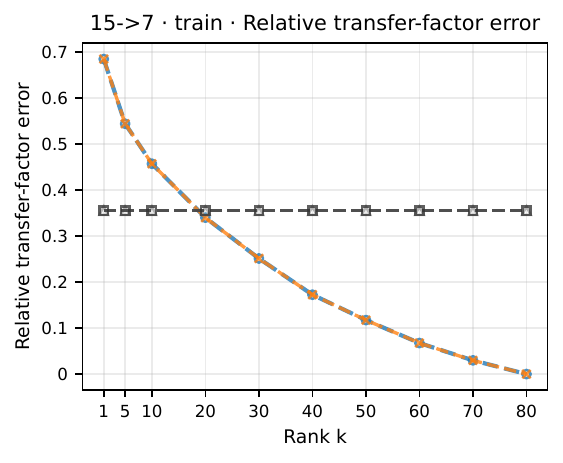}
        \caption{Transfer-factor fit error.}
        \label{fig:4_5_shared_map_updates:tau}
    \end{subfigure}

    \caption{
    In-sample fit quality of the top-\(k\) source-subspace shared map from \(\mathcal G^{(15)}\) to \(\mathcal G^{(7)}\) for the \(m=16\) prompt-subset condition.
    Increasing \(k\) captures more source-update energy, improves per-example directional agreement, and reduces matrix-reconstruction and stable-rank-transfer errors.
    At \(k=20\), the selected source subspace captures \(75.2\%\) of source-update energy.
    The relative Frobenius and spectral errors are \(0.639\) and \(0.405\), the mean row-wise cosine similarity is \(0.691\), and the relative transfer-factor fit error is \(0.340\).
    The more gradual energy-capture curve indicates that substantial source-update energy lies beyond the first 20 principal directions.
    The \(k=80\) endpoint is the in-sample interpolation limit and is shown only as a reference, not as evidence of a low-rank mechanism.
    }
    \label{fig:4_5_shared_map_updates}
\end{figure}

\begin{figure}[h]
    \centering
    \begin{subfigure}[t]{0.31\linewidth}
        \centering
        \includegraphics[width=\linewidth]{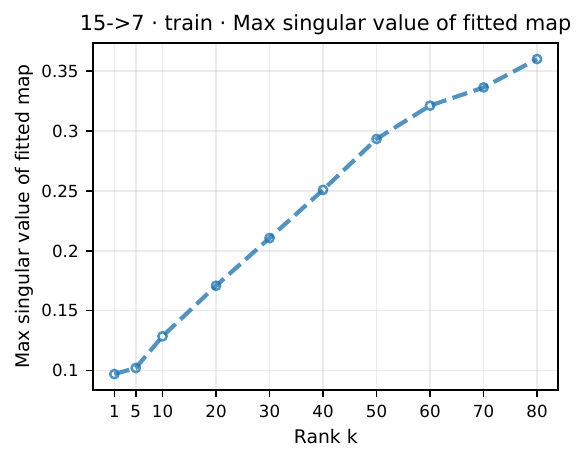}
        \caption{Largest singular value.}
        \label{fig:4_5_shared_map_updates_spectrum:smax}
    \end{subfigure}
    \begin{subfigure}[t]{0.31\linewidth}
        \centering
        \includegraphics[width=\linewidth]{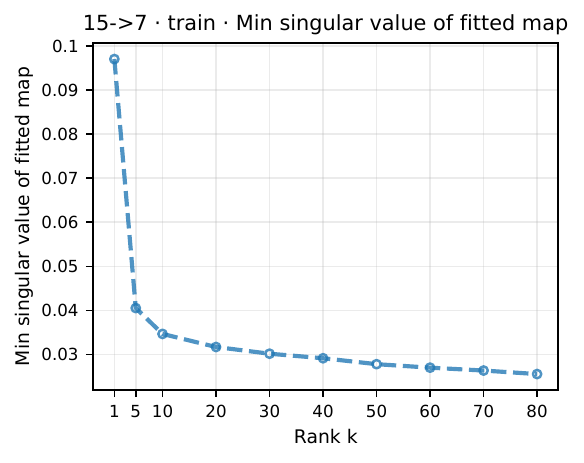}
        \caption{Smallest nonzero singular value.}
        \label{fig:4_5_shared_map_updates_spectrum:smin}
    \end{subfigure}
    \begin{subfigure}[t]{0.31\linewidth}
        \centering
        \includegraphics[width=\linewidth]{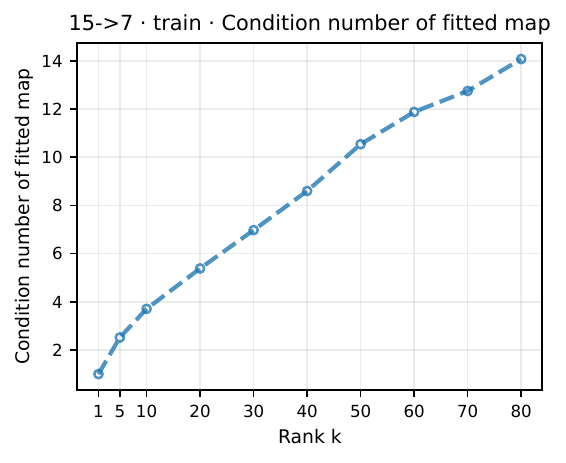}
        \caption{Reduced-map condition number.}
        \label{fig:4_5_shared_map_updates_spectrum:kappa}
    \end{subfigure}

    \caption{
    Singular-value diagnostics for the reduced coordinate map fitted to the gradient-induced activation updates.
    Between \(k=10\) and \(k=20\), the smallest nonzero singular value decreases from \(0.0347\) to \(0.0317\), while the largest singular value grows from \(0.129\) to \(0.171\).
    Consequently, the reduced-map condition number rises from \(3.71\) to \(5.38\).
    The rank-20 map improves the in-sample fit with substantially less conditioning deterioration than the corresponding raw-gradient map.
    }
    \label{fig:4_5_shared_map_updates_spectrum}
\end{figure}

\clearpage

\section{Diverse-refusal fine-tuning raises refusal residuals rank and weakens ablation - additional plots}

\begin{figure}[h]
    \centering
    \begin{minipage}{\linewidth}
        \centering

        \makebox[0.42\linewidth][c]{\textbf{8 refusal starts}}
        \hfill
        \makebox[0.42\linewidth][c]{\textbf{16 refusal starts}}

        \vspace{0.15em}

        \includegraphics[width=0.42\linewidth]{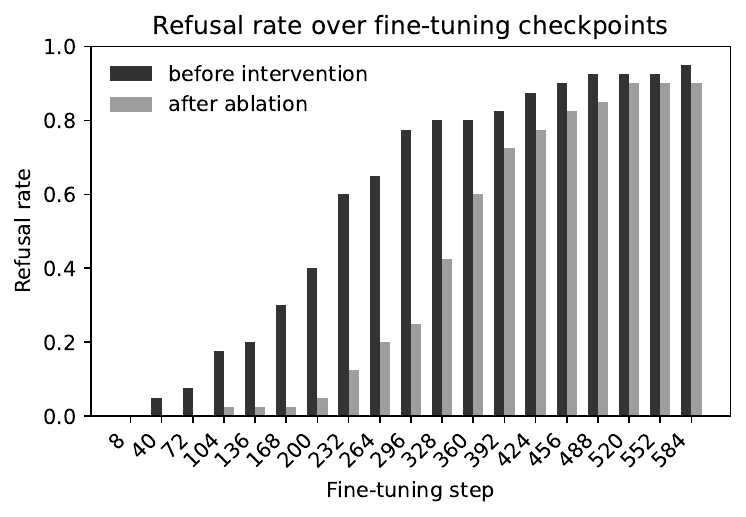}
        \hfill
        \includegraphics[width=0.42\linewidth]{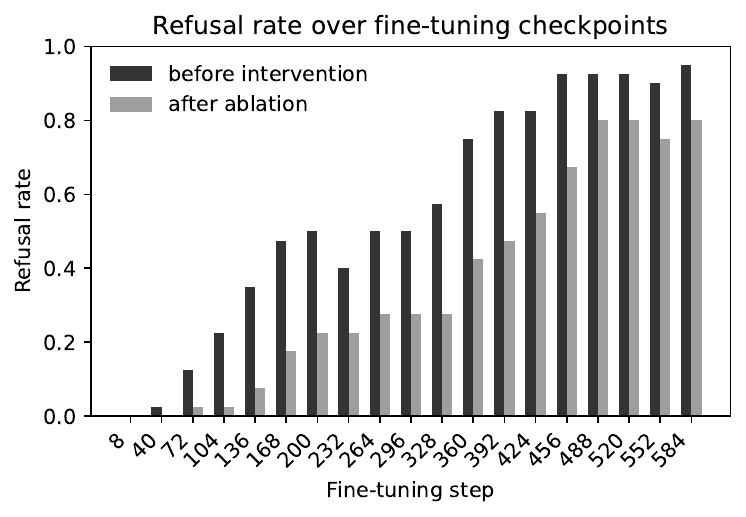}

        \par
        \makebox[0.42\linewidth][c]{\small (a) Refusal rates, 8 starts}
        \hfill
        \makebox[0.42\linewidth][c]{\small (b) Refusal rates, 16 starts}

        \vspace{0.35em}

        \includegraphics[width=0.42\linewidth]{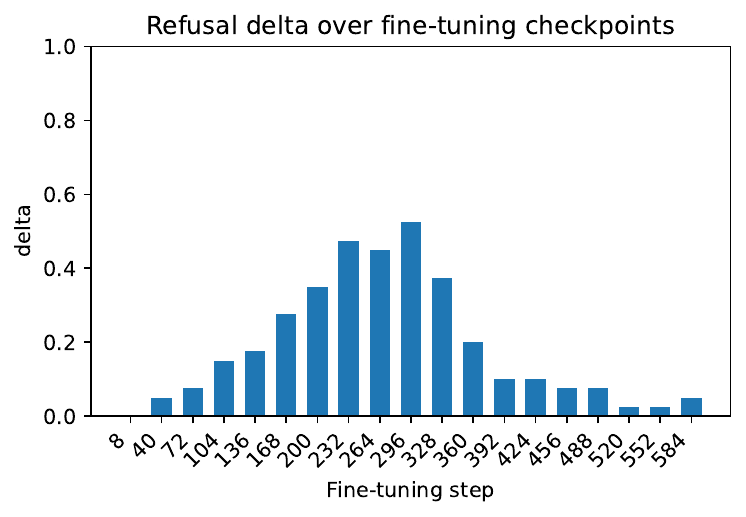}
        \hfill
        \includegraphics[width=0.42\linewidth]{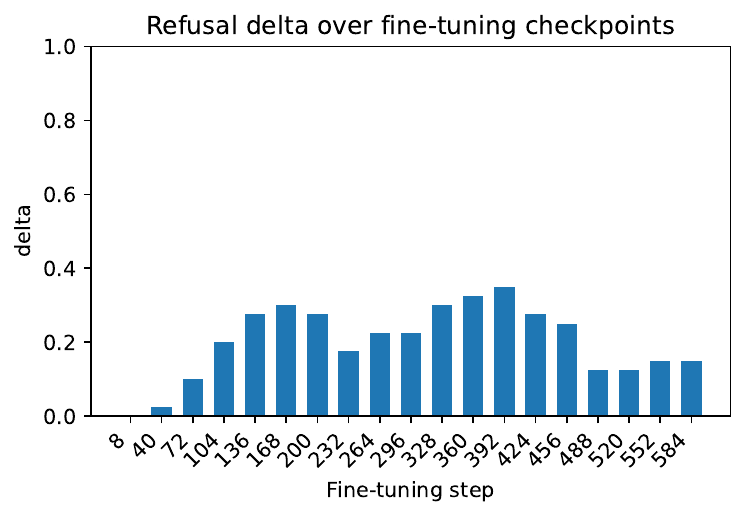}

        \par
        \makebox[0.42\linewidth][c]{\small (c) Absolute refusal delta, 8 starts}
        \hfill
        \makebox[0.42\linewidth][c]{\small (d) Absolute refusal delta, 16 starts}

        \vspace{0.35em}

        \includegraphics[width=0.42\linewidth]{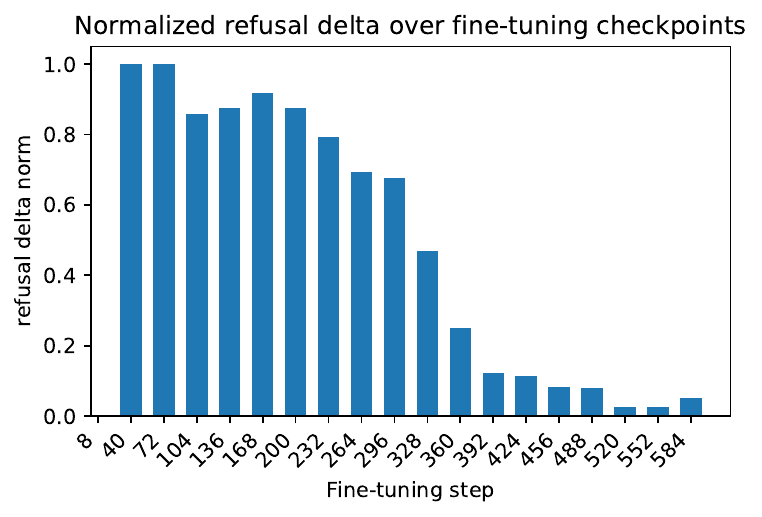}
        \hfill
        \includegraphics[width=0.42\linewidth]{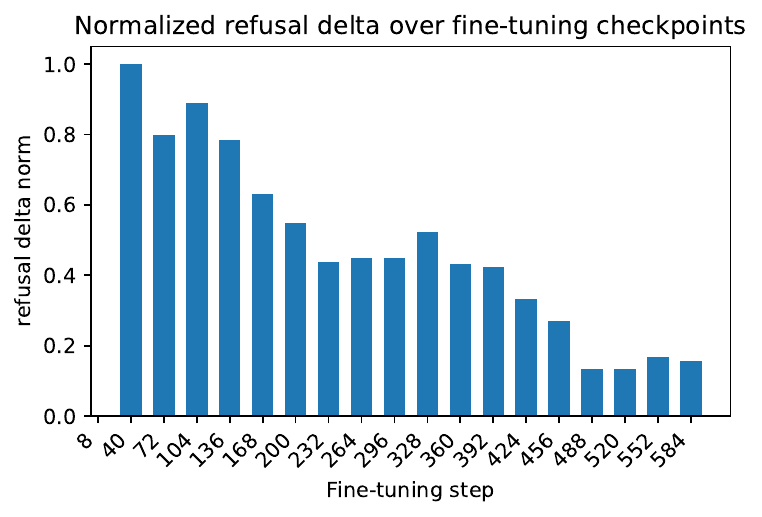}

        \par
        \makebox[0.42\linewidth][c]{\small (e) Normalized refusal delta, 8 starts}
        \hfill
        \makebox[0.42\linewidth][c]{\small (f) Normalized refusal delta, 16 starts}

        \caption{Optimization-step-matched refusal ablation trajectories for models trained with 8 and 16 refusal starts. Left and right columns show the 8- and 16-start conditions, respectively. The top row reports baseline and post-ablation refusal rates; the middle row reports the absolute refusal deltas; and the bottom row reports the fraction of baseline refusals removed by ablation.
        Ablation vulnerability exists in a small window of training since the refusal delta is peaked at first when maximum refusal rate is reached, decaying afterwards.
        The 8-start plot shows deltas having a single peak with maximum of 0.55, while 16-start plot shows two smaller peaks and reaches overall maximum of 0.375, consistent with reduced ablation attack effectiveness.
        }
        \label{fig:refusal_diversity_equal_step_trajectories}
    \end{minipage}
\end{figure}

\begin{figure}[h]
    \centering
    \begin{subfigure}[b]{0.34\linewidth}
        \centering
        \includegraphics[width=\linewidth]{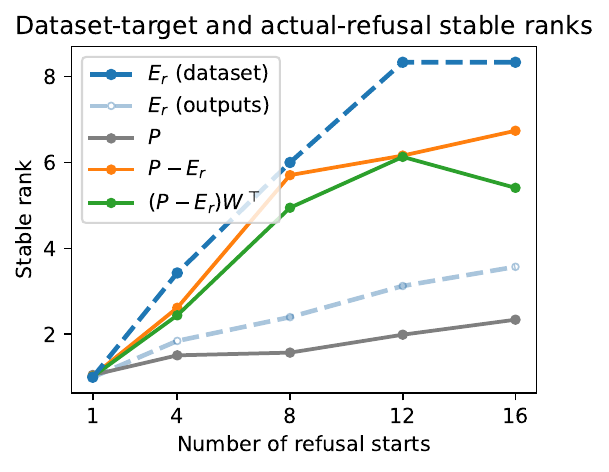}
        \label{fig:4_6_rank_propagation:head}
    \end{subfigure}
    \begin{subfigure}[b]{0.30\linewidth}
        \centering
        \includegraphics[width=\linewidth]{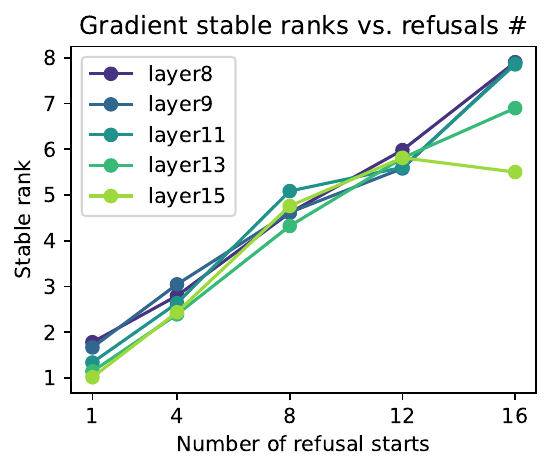}
        \label{fig:4_6_rank_propagation:layers_raw}
    \end{subfigure}
    \begin{subfigure}[b]{0.34\linewidth}
        \centering
        \includegraphics[width=\linewidth]{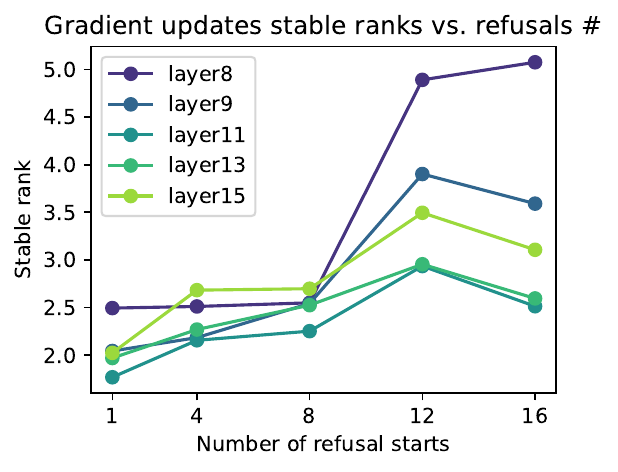}
        \label{fig:4_6_rank_propagation:layers}
    \end{subfigure}
    \begin{subfigure}[b]{0.75\linewidth}
        \centering
        \includegraphics[width=\linewidth]{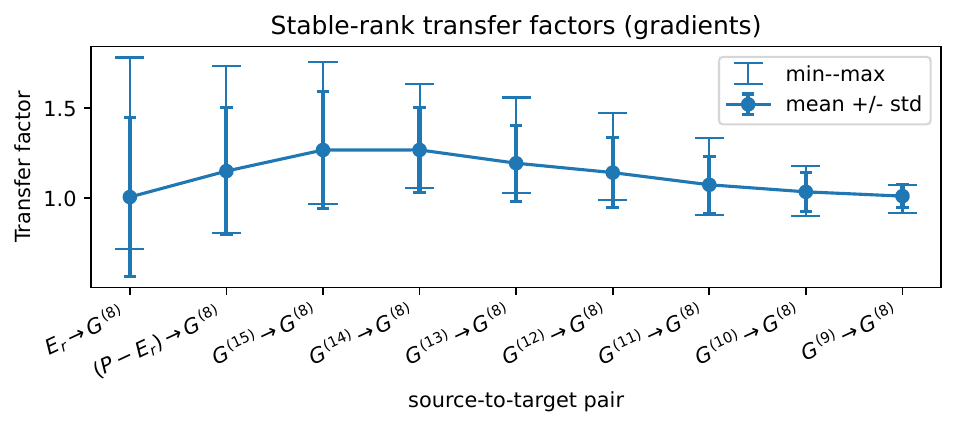}
        \label{fig:4_6_transfer_factors:transfer_raw}
    \end{subfigure}
    \begin{subfigure}[b]{0.75\linewidth}
        \centering
        \includegraphics[width=\linewidth]{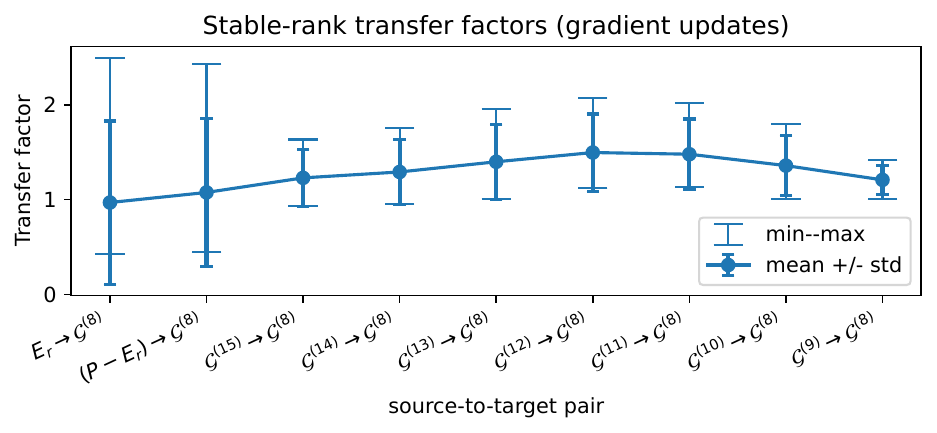}
        \label{fig:4_6_transfer_factors:transfer}
    \end{subfigure}
    \caption{
    Stable ranks and stable rank propagation under refusal-start diversification by fine-tuning for gradient and gradient-induced activation deltas.
    Top left: stable ranks of the dataset-target \(E_r\), realized-output \(E_r\), prediction matrix \(P\), target residual \(P-E_r\), and resulting last-layer gradient \((P-E_r)W^\top\).
    Top middle: stable ranks of the gradients \(G^{(l)}\).
    Top right: stable ranks of the gradient-induced activation updates \(\mathcal G^{(l)}\).
    Bottom rows: end-to-end stable rank transfer factors measured for raw gradients and gradient-induced activation deltas transfer factors.
    The target-space, gradients, and gradient-updates stable ranks increase overall as refusal-start support broadens, and the both gradient and gradient-update matrices undergo mild stable-rank expansion from layer 15 to layer 8 on average.
    }
    \label{fig:4_6_transfer_factors}
\end{figure}

\end{document}